\documentclass[12pt]{elsarticleelsarticle}

\usepackage{changepage}
\usepackage{amssymb}
\usepackage{amsmath}
\usepackage{verbatim}
\usepackage{makecell}
\usepackage{algorithm}
\usepackage{tikz}
\usepackage{algpseudocode}
\usepackage{pgfplots}
\pgfplotsset{compat=1.18}
\usepackage{graphicx}
\usepackage{subcaption}
\usepackage{float}
\usepackage{textcomp}
\usepackage{xcolor}
\newtheorem{theorem}{Theorem}
\newtheorem{lemma}{Lemma}

\newtheorem{definition}{Definition}

\newtheorem{proof}{Proof}

\journal{Computers and Electrical Engineering}

\begin{document}

\begin{frontmatter}


\title{UAV Path Planning for Object Observation with Quality Constraints: A Dynamic Programming Approach\\
}


\author[seu]{Jiawei Wang}
\author[seu]{Weiwei Wu}
\author[ams]{Yijing Wang}
\author[seu]{Yan Lyu}
\author[seu]{Vincent Chau}
\affiliation[seu]{organization={Southeast University, Department of Computer Science and Engineering},
            city={Nanjing},
            country={China}}

\affiliation[ams]{organization={Consulting Center for Strategic Assessment, Academy of Military Science},
            city={Beijing},
            country={China}}

\begin{abstract}
This paper addresses a UAV path planning task that seeks to observe a set of objects while satisfying the observation quality constraint. A dynamic programming algorithm is proposed that enables the UAV to observe the target objects with the shortest path while subjecting to the observation quality constraint. The objects have their own facing direction and restricted observation range. With an observing order, the algorithm achieves $(1+\epsilon)$-approximation ratio in theory and runs in polynomial time. The extensive results demonstrate that the algorithm produces near-optimal solutions, the effectiveness of which is also tested and proved in the Airsim simulator, a realistic virtual environment.
\end{abstract}


\end{frontmatter}

\section{Introduction}

With recent advancements in microchip and sensor technology, Unmanned Aerial Vehicles (UAVs) have become increasingly popular due to their diverse applications, including structure inspection, smart farming, wildfire detection, cinematography, etc. As an intelligent and integrated platform, UAVs can perform difficult or dangerous tasks for humans. For instance, an UAV can fly around a building and meticulously reconstruct it using photos taken along a pre-planned path, which is challenging for human operators to accomplish~\cite{uavstructureinspection}. UAVs equipped with specific devices can automate farmland irrigation in large areas~\cite{smart-farm}, hence liberating farmers from laborious work. Moreover, UAVs can perform missions in disaster-stricken areas~\cite{fire-detect} and capture stunning aerial photography from locations that are inaccessible to humans~\cite{uav-cinema}.

Benefiting from their flexibility and maneuverability in practice, UAVs generally require a trajectory to accomplish various tasks during flight. Coverage path planning algorithms are specifically designed to determine feasible routes for UAVs to follow, covering either a specific area or a set of Points of Interest (PoI), while adhering to certain conditions, such as distance constraints or hardware parameters, such as camera angles. Distance constraints refer to the UAV's perception range, supported by cameras or radars, which require the target object to be located within a certain distance of the UAV. A wide range of algorithms has been proposed in the literature to enable UAVs to cover a specific space, whether by planning feasible routes to cover PoIs (Points of Interest) or by aggregating the area swept over during flight.

This study delves into a realistic scenario where a UAV must observe a set of objects while meeting a gross observation quality constraint, distinguishing it from a point cover problem. The problem can be extended to building inspection, where a path guides the UAV to capture photos of each side of the building and, meanwhile, considers the quality of the photos. However, manual planning of a route that is as short as possible can be challenging. To the best of our knowledge, this work stands out as the first to propose a solution with observation quality and an approximation guarantee for UAV path planning, distinct from other literature that seldom considers the coverage quality and provides no theoretical bound on the path length. Our proposed solution improves the efficiency of UAV path planning, benefiting various industries and applications.

\section{Related Work} 
\label{related-work}

As this paper plans a path for a UAV in covering problems, we present related works in three parts. The first part is devoted to different methods of UAV path planning. The focus is then shifted to the UAV covering problem, in which the main task is to find a path that achieves covering effects. Finally, we briefly introduce several classical Traveling Salesman Problem (TSP) algorithms. 

\subsection{UAV Path Planning}
Path planning is critical to unmanned aerial vehicle control, enabling UAVs to autonomously navigate along predetermined paths. Numerous methods have been proposed in literature, categorized into sampling-based algorithms, mathematical model-based algorithms, bio-inspired algorithms, and artificial intelligence algorithms.

Sampling-based algorithms require prior knowledge of the map and divide it into a set of nodes. The concrete techniques used in these algorithms include Rapid-Exploring Random Trees (RRT), RRT-star, and A-star. Mathematical methods primarily employed Lyapunov functions~\cite{Lyapunov} to maintain the stability of UAVs, linear programming to integrate all cost factors with Hamiltonian function to search for an optimal path, as described in \cite{math-planning}, and Bezier curves, which further consider the smoothness of the flight path. Bio-inspired techniques~\cite{bio-planning} utilized evolutionary ideas, selecting a route as a parent path and generating new paths through mutation and crossover. The selection process is guided by the adaptive value of the offspring.

AI-based methods have become a focal point of research in recent years, and these methods can be further classified into those based on traditional machine learning models, deep learning, and reinforcement learning. Supervised machine learning includes methods such as Gaussian filtering and Kalman filtering, which involve cleaning the data collected from the environment, inputting it into the model, and obtaining the predicted path and data visualization. Kang et al. in~\cite{Kalman-filter1} and the authors in~\cite{Kalman-filter2} used Kalman filters to clean noisy images collected by UAVs in complex environments, and obtained a safe flying space by outputting collision probabilities. Due to Kalman filters' limitations in recognizing numerous obstacles, however, the authors in ~\cite{Gaussian-filter} used Gaussian filters to learn the impact of UAV maneuvers and more accurately estimate the UAV's state. Unsupervised machine learning, such as the k-means method~\cite{k-means-planning}, determined the route by clustering the target points in a multi-task environment.

Computer vision is a representative direction of deep learning applied to UAV path planning. Similar to the visual technology used in autonomous driving cars, UAVs use cameras to capture real-time images and identify possible obstacles and safe flying spaces~\cite{NN-planning1}, ~\cite{NN-planning2}. In addition to deep learning, Q-learning in reinforcement learning was also used for path planning in~\cite{q-learning-planning}, and Luan et al. in~\cite{g-learning-planning} used G-learning to iteratively optimize the path by calculating the cost function in real time. Lei et al. in~\cite{DRL} modeled the UAV navigation problem as a Markov decision process and proposed a model interpretation method based on feature attributes to explain the behavior of the UAV in the path planning process. Similarly, Xie in ~\cite{DRL2} expressed path planning as a partially observable Markov process, constructing RNNs to tackle partially observable problems, and leveraging reward value and action value to reduce meaningless exploration. These two deep reinforcement learning methods significantly enhanced the learning efficiency of Q networks and G networks.

\subsection{UAV Covering}
As a subproblem of UAV path planning, UAV covering additionally requires achieving a cover task, as indicated in the name. One type is regional covering that assumes a region is covered when it falls within the perception area of the UAV. A relatively conventional algorithm is presented in~\cite{cellular} where Nam et al. utilized approximate cellular decomposition with criteria for both length and number of turns on the route. Yao et al.~\cite{river-rescue} proposed an offline route planning method based on Gaussian mixed model and heuristic prioritization, which was aimed at maximizing the probability of finding the lost target in a river rescue mission. Xie et al.~\cite{multi-polygon} solved the problem of multiple polygon regions by integrating covering a single polygon and traveling salesman problem. These are the methods for one UAV and algorithms built on multiple UAVs are as follows. Jing et al.~\cite{uavstructureinspection} navigated the UAVs to inspect a large complex building structure by combining the set covering problem and the vehicle routing problem. Rapidly exploring random tree, as indicated by its name, is broadly used to explore an unknown region with high efficiency. In~\cite{surveillance}, a cooperative surveillance task was achieved by multiple UAVs where RRT was modified to find feasible trajectories passing suitable observing locations, on which particle swarm optimization was then performed. Focusing on the coverage path planning of heterogeneous UAVs, the authors in \cite{hetero-cover} established the UAVs and regional model before using linear programming to accurately provide the best point-to-point flight path for each UAV. Then, inspired by the foraging behavior of ants, a heuristic algorithm based on ACS was proposed to search for the approximate optimal solution and to minimize the time consumption of tasks in the cooperative search system.

The other type is object-oriented where the problem aims to cover a set of objects, with the condition satisfied when the distance between the UAV and an object is less than a threshold. In this case, effective coverage can be treated as a binary in that 1 for covered and 0 for not covered, which can find many applications in the context of wireless sensor networks. Huang et al.~\cite{seuwsn} embedded turning angles and switching numbers during retrieving data into graph structure and obtained a path through the use of Generalized TSP solutions. Authors in~\cite{straightwsn} completed the data collection mission in a straight line situation, minimizing the UAV's total flight time via dynamic programming and meanwhile retrieving a certain amount of data each sensor. After determining the possible flying path, Yang et al.~\cite{bio-inspired} combined a genetic algorithm and ant colony optimization to select the optimal path for data collection. In~\cite{robotcover}, the agents viewed moving objects and obstacles as disks of different sizes and the goal was to find a collision-free coverage path in a dynamic changing environment while ensuring smoothness. Concerning multiple UAVs, Alejo et al.~\cite{wsn} integrated online RRT with genetic algorithms, guiding several UAVs collecting data simultaneously from sensors randomly distributed. A multi-agent architecture designed in~\cite{multi-agent-catastrophe} allowed UAVs to patrol in a region and monitor key ground facilities. 

\subsubsection{Traveling Salesman Problem (TSP)}
The Traveling Salesman Problem and its variants are thoroughly studied in literature. The problem is to find a route that starts and ends at the same place, such that a set of places need to be visited once. The authors in~\cite{DFJ-tsp, MTZ-tsp, GP-tsp} designed different forms of integer linear programming for the problem. Christofides~\cite{1.5TSP} obtained a 3/2 approximation by combining a minimum spanning tree of the original graph and the best matching of the vertices with odd degrees. Nearest neighbor algorithm~\cite{nearest-tsp} constructed the path by continuously adding the new vertex closest to the current one. Cheapest Insertion~\cite{cheapest-insertion} enlarged the route via inserting new node to the path with the lowest insertion cost and its approximation ratio is 2. Holland~\cite{genetic-tsp} employed genetic algorithms, nurturing route offspring with best adaptive score. Note that the TSP is a NP-hard problem, and therefore we are interested in finding approximate solutions in polynomial time.

\subsection*{Contributions}

The literature summarized above admittedly achieved satisfying outcomes in different circumstances, yet few of them truly yield theoretical results, suggesting the degree to which they are efficient relative to the optimal solutions. In our work, however, a result with an approximation ratio is obtained.

Our contributions are summarized as follows:
\begin{enumerate}

\item A Dynamic Programming based algorithm is designed which can solve the approximate shortest path given the quality constraint. 
\item An Integer Linear Programming (ILP) is designed which can compute the near-optimal solution via a solver with respect to epsilon.
\item It is theoretically guaranteed that the algorithm runs in polynomial time and achieves $1+\epsilon$ approximation given a visiting order.
\item We carry out various numerical and simulation experiments and the results show that our method produces near optimal solution and makes significant improvement on the observation quality.
\end{enumerate}

The structure of the rest of the paper is organized as follows. 
In Section \ref{sec: system-model}, the system model is presented followed by discretization and the formalization of the problem. The detail of our methods is elaborated on in Section \ref{sec: methodology} including the design of dynamic programming algorithm and determination of the observation orders. The process and results of the experiment are presented in Section \ref{sec: experiments}.

\section{System Model and Problem Formulation}
\label{sec: system-model}
In this section, we first present the description of the system model, followed by the formalization of the goal of the problem, which aims to find a route for the UAV that minimizes the total length of the flying route while ensuring the observation quality constraint and that all the objects have been observed. After that, we describe the process to generate the observation points via area discretization and finally present the integer linear programming form of the problem. 

\subsection{System Model}

Consider a two-dimensional space in which there are $n$ objects and a UAV equipped with a camera that can move freely in this space. The objects set is referred to as $\mathcal{O}= \{o_1, o_2, o_3, ..., o_n\}$. Each object $o_i$ faces a direction $\overrightarrow{d_{o_i}}$, the location of which is described by coordinates $(x_{o_i}, y_{o_i})$. In order to ensure the observation quality, the UAV is not allowed to take photos from an over-deviated orientation, which is supported by the fact that the most information is obtained when the camera is facing directly against an object. Moreover, since distance also affects the observing quality, the UAV cannot observe the object if the distance exceeds a threshold. Concretely, the UAV has to be located within an angle range around the object's direction in order to observe it, which is denoted as $\theta$. Besides, the maximum observation distance is defined to be $d_{max}$. To further ensure safety, the UAV cannot approach too close to the object, so the minimum distance is set to $d_{min}$. As illustrated in Figure~\ref{img:efficient_observation}, the UAV can efficiently observe an object if and only if it satisfies Definition~\ref{def:efficient_observe}. Dashed blue lines indicate the effective area. We have adapted this definition~\cite{deploy}  to underscore the emphasis on the quality of observation in our problem.

We define the observation quality of the object $o_i$ as $q_i$, thus the gross quality of all objects $Q = \sum_{i} q_i$. There is also an observation quality constraint $q^*$ that the gross quality of all objects $Q$ has to attain. In formula, $Q \geq q^*$. Formula~(\ref{formula:q_i}) is presented to calculate $q_i$ drawn from ~\cite{deploy}, which is short of $q(o_i, p_j)$, the observation quality of $o_i$ from the point $p_j$. Suppose that there exist maximum and minimum observation quality $q_{max}$, $q_{min}$ for a single object, which are decided by the observation area. To make the problem solvable, we introduce an additional assumption: $n*q_{min} \leq q^* \leq n*q_{max}$.

\begin{definition}[Efficient observation]\label{def:efficient_observe} 
An object $o_i$ is efficiently observed by the UAV $p_j$ if and only if $d_{min} \leq ||{o_i}{p_j}|| \leq d_{max}$, $\alpha(\overrightarrow{d_{o_i}}, \overrightarrow{o_i p_j }) \leq \theta$, where $\alpha( , )$ the angle between two vectors.
\end{definition}

We define the observation quality of the object $o_i$ as $q_i$, thus the gross quality of all objects $Q = \sum_{i} q_i$. There is also an observation quality constraint $q^*$ that the gross quality of all objects $Q$ has to attain. In formula, $Q \geq q^*$. Formula~(\ref{formula:q_i}) is presented to calculate $q_i$ drawn from ~\cite{deploy}, which is short of $q(o_i, p_j)$, the observation quality of $o_i$ from the point $p_j$. Suppose that there exist maximum and minimum observation quality $q_{max}$, $q_{min}$ for a single object, which are decided by the observation area. To make the problem solvable, we introduce an additional assumption: $n*q_{min} \leq q^* \leq n*q_{max}$.

\begin{equation}\label{formula:q_i}
q_i=\left\{\begin{aligned}
&\frac{a}{(||o_ip_j||+b)^2}cos(\alpha(\overrightarrow{d_{o_i}}, \overrightarrow{o_i p_j })),\\
&\ \ \ \ \mbox{if}\  d_{min} \leq ||o_ip_j|| \leq d_{max} \mbox{ and }  \alpha(\overrightarrow{d_{o_i}}, \overrightarrow{o_i p_j }) \leq \theta \\
&0, \ otherwise.
\end{aligned}
\right.
\end{equation}

\begin{figure}[htb]
\centerline{\includegraphics[width=0.4\textwidth]{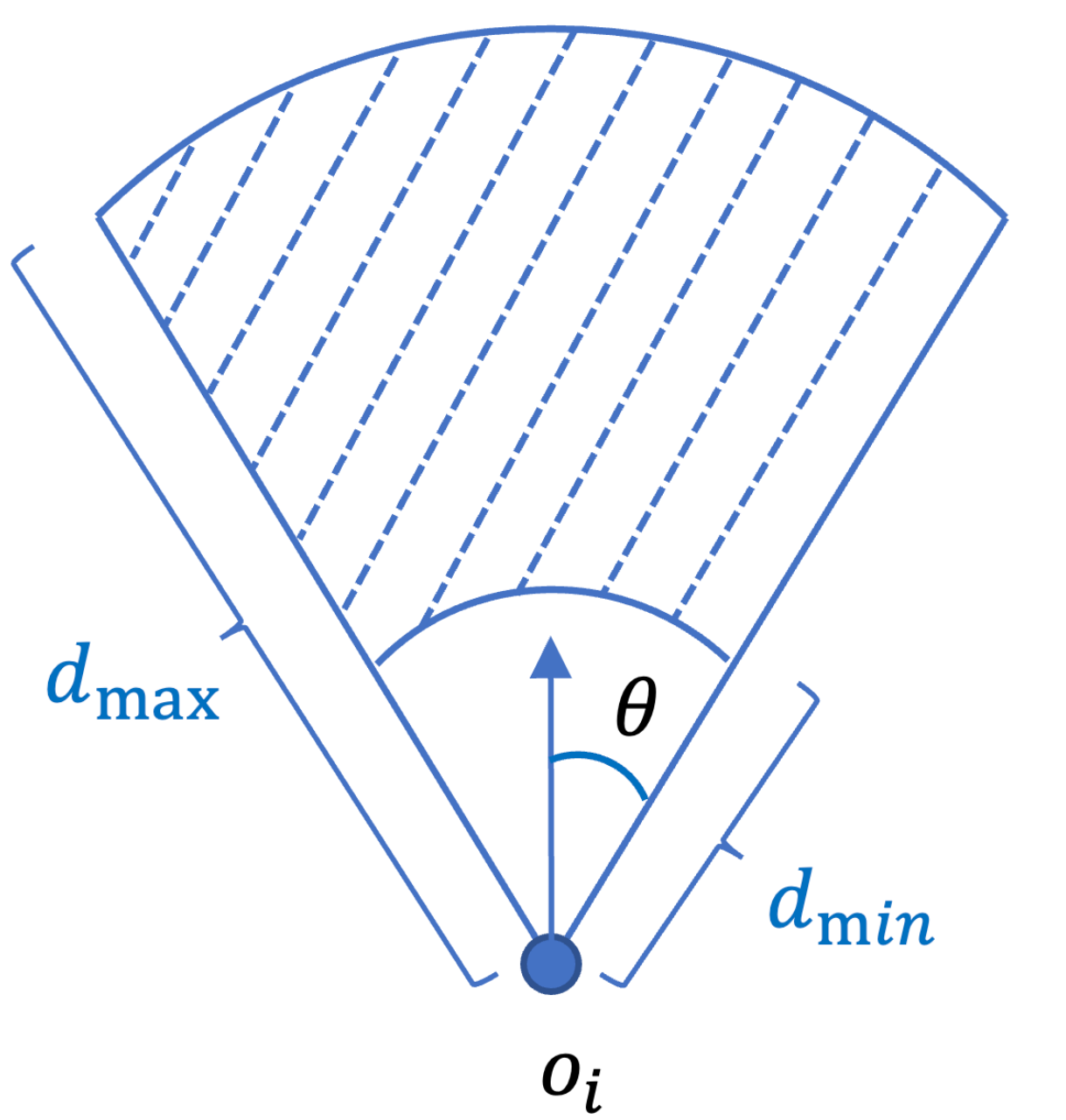}}
\caption{Efficient observation. The UAV cannot observe an object from beyond $d_{max}$, or closer than $d_{min}$, or deviated over an angle $\theta$.}
\label{img:efficient_observation}
\end{figure}

Given that a single UAV is tasked with covering all specified objects from a starting point $P_0$, the objective is to minimize the path length while adhering to quality requirements. However, this poses a challenge as the problem cannot be solved directly through a Traveling Salesman Problem approach, which does not account for the unique characteristics of the observation points and their associated qualities. In certain cases, there may be instances where observation areas of multiple objects overlap, allowing the UAV to capture images of multiple objects by rotating its camera at certain points. Due to the arbitrary nature of the object layout, devising a comprehensive solution that selects suitable observation points and plans the path remains a challenging task.

\subsection{Discretization of the observation points} \label{sec:discrete_points}

Provided that the objects are located in a continuous space, a method is first designed to generate discrete observation points which enables the UAV to practically plan its flying path in the finite space. It is inevitable that discretization brings about $(1+\epsilon)$ approximation error~\cite{approximation-algorithm}. We apply the method in~\cite{deploy} to our problem and it is briefly summarized below. Figure~\ref{img:discrete-area} shows a possible discretization result of a target area.

It is intended to divide the whole effective observation area into several independent segments so that the observation quality in each segment is considered the same. Therefore, it only needs to choose the upper left corner point of a segment to represent it. The method begins by dividing the area in the distance domain from the object into $K_1$ rings, with end points referred to as $l(0), l(1), ..., l(K_1)$. Each ring is further separated into $K_2$ segments on both sides of $\overrightarrow{d_{o_i}}$, with endpoints in the angle domain defined as $a(0), ..., a(K_2)$. In Figure~\ref{img:discrete-area}, there are $K_1 =3$ rings and there are $K_2 = \{ 1, 2, 3\}$ respectively for each sector ring. A point, for example, can be expressed by $(l(2), \alpha(1))$. With this method, it is ensured that the error of observation quality is associated with $\epsilon$. Apart from this, we make the following revision since the granularity may not be fine enough for the same approximation error of path length.

Let $A_1 = l(k_1+1)-l(k_1)$ refer to the length of a segment, and $B_1 = \alpha(k_2+1)-\alpha(k_2)$ refer to the angular width of the segment. Let $D$ denote the maximum distance between two objects from the set $\mathcal{O}$. Formally, $D=\max _{i, j \in |\mathcal{O}|} d\left(o_i, o_j\right)$. We use Equation~(\ref{formula:mesh}) to determine the size of a segment and ensure the approximation ratio, with $n$ being the number of objects.

\begin{equation}
    A = min\{A_1, A_2\}
\end{equation}

\begin{equation}
    B = min\{B_1, B_2\}
\end{equation}

\begin{equation}
    \delta = A_2 =B_2 = \frac{\epsilon * D}{n} \label{formula:mesh}
\end{equation}

\begin{figure}[htbp]
\centerline{\includegraphics[width=0.4\textwidth]{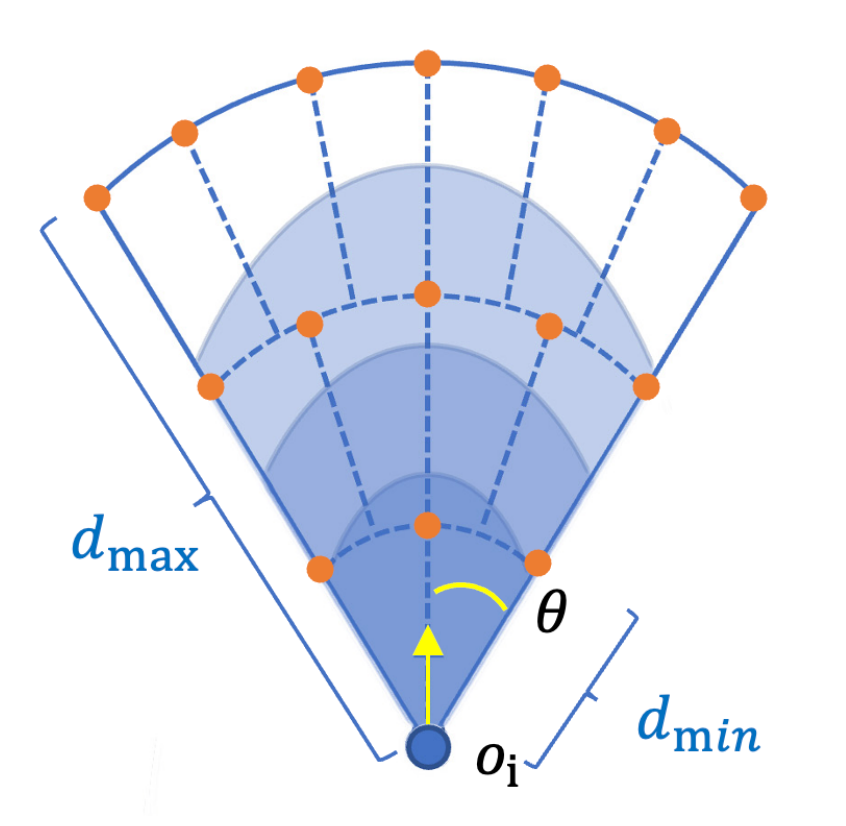}}
\caption{Discretization result. The observation quality within each segment is considered equal. $\theta$ is the observing angle range around $o_i$.}
\label{img:discrete-area}
\end{figure}

The two proposed equations are then used to generate all feasible segments and observation points in the space. The process is described in Algorithm~\ref{algo:generate_observation_points}. For convenience, we suppose there is an object $o_0$ centered at the original point and facing the positive x-axis direction ($\overrightarrow{d_{o_0}}=0$). The discretization method is then applied to $o_0$ that produces the observation point set $P_*$. For each object $o_i$, its corresponding observation points $P_i$ can be obtained by rotating all the points in $P_*$ to the direction of $o_i$ before translating to $(x_{o_i}, y_{o_i})$. The aggregation of $P_i$ is $\mathcal{P}$ that contains $\{p_1,\, p_2,\, p_3,\, ... ,\, p_u\} $ and the coordinate of the $i^{th}$ point is $(x_{p_i},\, y_{p_i})$. The observation points of an object are illustrated in Figure~\ref{img:discrete-area}, in which the orange dots are the observation points and the magnitude of observation quality can be indicated by the different blue color.

\begin{algorithm}[ht]
  \caption{Generation of observation points} 
    \label{algo:generate_observation_points}
    \begin{algorithmic}[1]
      \Require
        Object set $\mathcal{O}$
      \Ensure
        $\mathcal{P}$: Observation points 
        \State $\mathcal{P} \gets \emptyset$
        \State Given an object $o_0$ with $\overrightarrow{d_{o}}=0$ and coordinates $(0, 0)$, generate observation points set $P_*$ via discretization process. 
        \For{object $o_i$ in $\mathcal{O}$}
            \State Rotate $P_*$ to direction $\overrightarrow{d_{o_i}}$ and translate to $(x_{o_i}, y_{o_i})$. 
            \State Define the result as $P_i$.
            \State $\mathcal{P} \gets \mathcal{P} \cup P_i$.
        \EndFor
        \State return $\mathcal{P}$
    \end{algorithmic}
\end{algorithm}

\begin{lemma}\label{lemma:mesh}
The gross error of the discretization process does not exceed $\epsilon*D$.
\end{lemma}

\begin{proof}
Given that observation points are utilized for approximating the points in a 2D continuous space, each point's approximation error is limited to at most $\delta$ when rounding it to the nearest mesh point. Additionally, the total error is restricted to $n*\delta$, as no more than $n$ observation points are chosen in the algorithm, which will be elaborated later. It can be anticipated that the total error does not surpass $\epsilon$ times the lower bound of the optimal solution, ensuring that the actual relative error is at most $\epsilon$. As the UAV needs to travel at least this distance to observe two objects efficiently, and the situation with only one object is not considered, it is evident that $D$ fulfills the requirement. Thus, Lemma~\ref{lemma:mesh} is established. By setting $n*\delta = \epsilon*D$, Equation~(\ref{formula:mesh}) is derived.
\end{proof}

\subsection{Integer Linear Programming representation}
\label{sec:ILP}
In this section, we present the problem in an integer linear programming (ILP) form.

Based on the discrete observation points discussed in Section~\ref{sec:discrete_points}, an Integer Linear Programming (ILP) approach is proposed. The approach is based on the concept of dividing the map into independent zones, each comprising an object's effective observing area and observation points capable of covering it. The starting point is considered an independent zone that contains only a single point. The subsequent step involves selecting a single point from each zone to form a route that passes through all observation points and the starting point. The formulated ILP approach is based on the Traveling Salesman Problem (TSP) formulation and is illustrated in Equation~(\ref{eq:0}) to Equation~(\ref{eq:6}).

\begin{align}
    &\min \sum d_{i, j, p_1, p_2}\cdot X_{i, j, p_1, p_2}  \label{eq:0} \\ 
    &s.t.\sum_{j}\sum_{p_1}\sum_{p_2} X_{i, j, p_1, p_2} = 1, &\forall i \in [1, n+1] \label{eq:1}\\
    &\sum_{i}\sum_{p_1}\sum_{p_2} X_{i, j, p_1, p_2} = 1, &\forall j \in [1, n+1] \label{eq:2}\\
    &\sum_{j}\sum_{p_2} X_{i, j, p_1, p_2} = \sum_{j}\sum_{p_2} X_{j, i, p_2, p_1} &\forall i \in [1, n+1], p_1\in zone 1 \label{eq:3}\\
    &\sum_{j \notin \{1, N+1\}}\sum_{p_2} (X_{1, j, p_1, p_2}+X_{N+1, j, p_1, p_2}) \notag\\
    &\quad= \sum_{j \notin \{1, N+1\}}\sum_{p_2} (X_{j, 1, p_2, p_1}+X_{j, N+1, p_2, p_1} )  &\forall p_1\in zone 1  \label{eq:4}\\
    &u_{i}-u_{j} + NX_{ijp_1p_2} \leq N-1,  &\forall i, j \in V, i\neq j \neq 0 \label{eq:5}\\
    &\sum_{i}\sum_{p_1}q_{i, p_1}\sum_{j}\sum_{p_2} X_{i, j, p_1, p_2} \geq q* \label{eq:q_constrain}\\
    &X_{i, j, p_1, p_2} \in 	\left\{ 0, 1 \right\}, u_{i} \geq 0, u_{i} \in R \label{eq:6}
\end{align}
\label{ILP_formula}

Let $X_{i, j, p_1, p_2}$ be defined as 1 if an edge exists between $p_1$ in zone $i$ and $p_2$ in zone $j$. The TSP route is then constructed by selecting all edges for which $X_{i, j, p_1, p_2} = 1$. $d_{i, j, p_1, p_2}$ represents the distance between point $p_1$ in zone $i$ and $p_2$ in zone $j$. The objective function~(\ref{eq:0}) aims to minimize the total cost of selected edges, subject to the constraints presented as follows. Constraint~(\ref{eq:1}) ensures that there is only one incoming edge coming from other zones connecting to only one point in zone $i$, while constraint~(\ref{eq:2}) guarantees that there is one outgoing edge leaving zone $i$ and heading to some other zone. To ensure that only one point is selected in each zone, the point connected with the incoming edge must be identical to the one connected with the outgoing edge. This requirement is formulated in constraint~(\ref{eq:3}), where the in-degree is always equal to the out-degree. Thus, a point is either not selected or has both an incoming edge and an outgoing edge. Furthermore, constraint~(\ref{eq:4}) ensures the correctness of the degree of the first and last point in the path. In addition, an effective TSP path requires the absence of subloops, which is achieved by modifying the Miller-Tucker-Zemlin (MTZ) constraint~\cite{MTZ-tsp}, as presented in constraint~(\ref{eq:5}). Finally, Equation~(\ref{eq:q_constrain}) is used to ensure the observation quality constraint, where $q_{i, p_1}$ denotes the quality obtained from point $p_1$ in the $i^{th}$ zone.

\section{Methodology}
\label{sec: methodology}
This section includes the lower bound and complete methods for the problem. The general idea to solve the problem is using dynamic programming to adjust an obtained path, which is able to observe all the objects before returning to the starting point but may not be subject to the quality constraint. We first explain the design of dynamic programming as an adjustment that can give a $(1+\epsilon)$-approximation solution of any given visiting order. The methods to search for the visiting orders are then shown before presenting the lower bound of the problem.

\subsection{Dynamic programming adjustment}\label{sec: DP}

 A dynamic programming algorithm, the core part of our work, is presented at length in this section. More specifically, it processes from the first object to the last one in order by continuously updating the shortest distance from the starting point to the current point, as the relaxation does in Dijkstra's algorithm~\cite{Dijkstra}. The quality constraint necessitates recording the observation quality at each point along the path, which leads to the design of a multidimensional table. Furthermore, the shortest path to $o_i$ can be exactly a straight line from that of $o_{j}$ if there is an overlap of the observation areas of objects $o_{j+1}$ to $o_i$ so that one observation point may be sufficient for the UAV to cover all of them (Figure~\ref{img:overall-img}).

\begin{figure}[htbp]
\centerline{\includegraphics[width=0.7\textwidth]{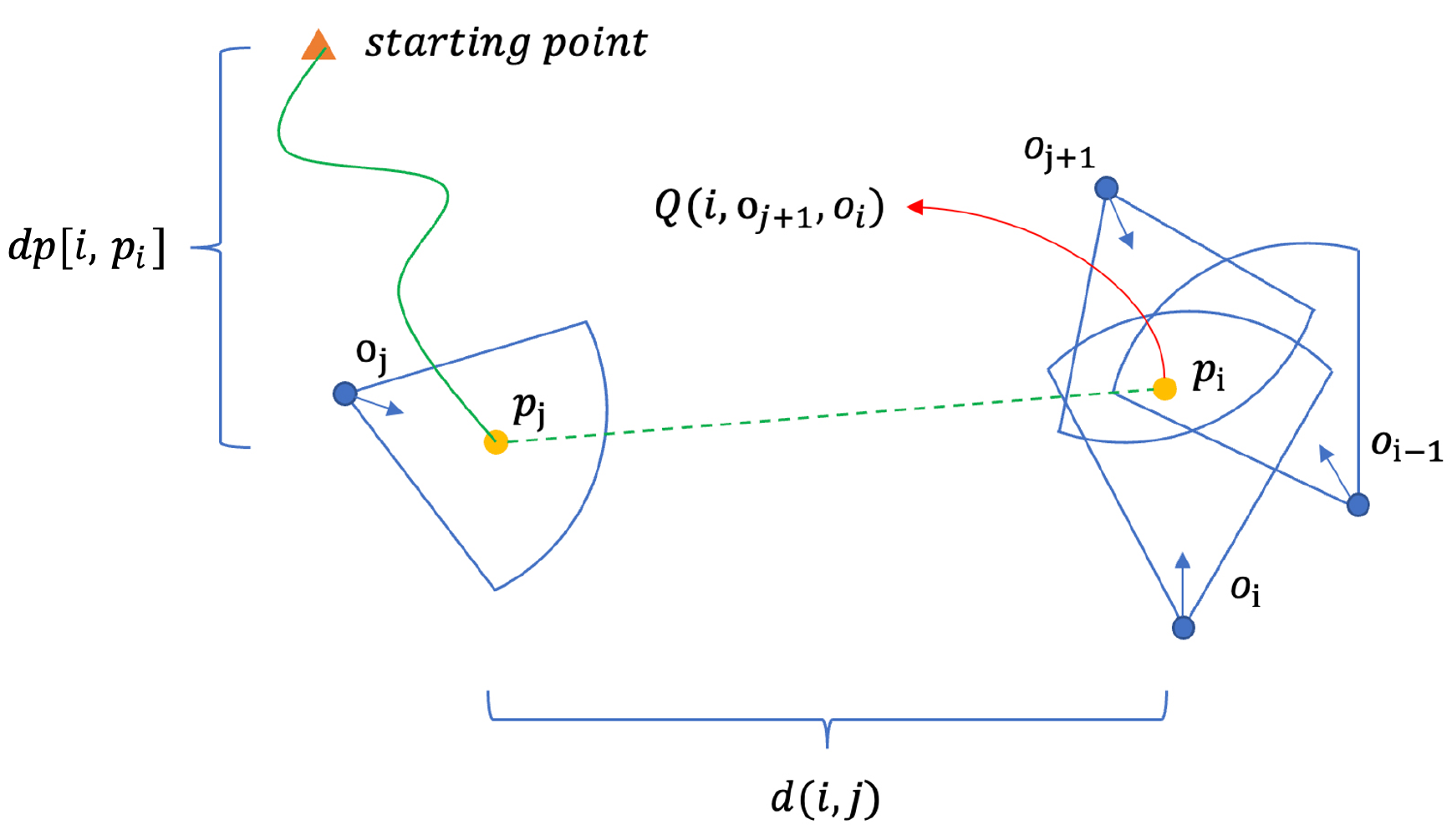}}
\caption{Construction of the path from $o_j$ to $o_i$ with the quality $q$. Since the UAV can observe $o_{j+1}$ to $o_{i}$ at $p_i$, it can directly go from $p_j$ to $p_i$, as the green dashed line shows.}
\label{img:overall-img}
\end{figure}

\begin{definition}
    Let $dp[i, p_i]$ be the set containing pairs of the path length and observation quality $(d, q)$, which implies one path to observation points $p_i$ of the $i^{th}$ objects. Formally, $dp[i, p_i] = \{(d_1, q_1), (d_2, q_2), ..., (d_z, q_z)\}$.
\end{definition}

Note that $i$ here is the index of the object in the visiting order. If the order is $[0, 1, 3, 4, 2, 0]$ in which $0$ refers to the starting point, the second object to observe is $o_3$. 

\begin{definition}
    Let $Q(p_i, o_{j+1}, o_i)$ be the gross observation quality of $p_i$ that observes the object $\{o_{j+1}, ..., o_i\}$ with respect to the observation order. If there is any object not observed by $p_i$, $Q(p_i, o_{j+1}, o_i)$ is set to negative infinity.
\end{definition}

According to the above definitions, the dynamic programming recursion is designed as follows.

\begin{align} \label{formula:DP}
    dp[i, p_i] = \mathop{\min}_{\substack{{j<i} \\ p_j \in \mathcal{P}}} { dp[j, p_j]+ (d(p_j, p_i), Q(p_i, o_{j+1}, o_i))} 
\end{align}

In Equation~(\ref{formula:DP}), $d(p_j, p_i)$ is the distance from $p_j$ (belonging to object $o_j$) to $p_i$ (belonging to object $o_i$). From an overall perspective, the form of Equation~(\ref{formula:DP}) is similar to that of the shortest path problem, yet the most distinctive feature is that a pair $(d_z, q_z)$ is used to record the path length $d_z$ from $P_0$ to the current object $i$ when the UAV acquires quality $q_z$ on this segment of the path. 

Since $dp[i, p_i]$ is the set containing all the paths feasible to point $p_i$ of object $o_i$, there exist many paths that are dominated. Therefore, it is necessary to eliminate all the dominated pairs in $dp[i, p_i]$ after each updating via Equation~(\ref{formula:DP}) based on Definition~\ref{def:domiante}, denoted as $\min$.

\begin{definition}\label{def:domiante}
    Suppose there are two pairs $a = (d_1, q_1)$ and $b = (d_2, q_2)$. Define $a$ dominates $b$ if it satisfies any of the following conditions:
    $(1):d_1=d_2$ and $q_1>q_2$, $(2):d_1<d_2$ and $q_1\geq q_2$,
    $(3):d_1=d_2$ and $q_1= q_2$.
\end{definition}

With the transition equation, the objective function is shown in Equation~(\ref{formula:objective of DP}) and the initial condition in Equation~(\ref{formula:initial condition of DP}).

\begin{equation} \label{formula:objective of DP}
    \mathop{\min}_{\substack{p_n \in \mathcal{P}}} { dp[n, p_n] + (d(p_n, P_0), 0)} \\
\end{equation}

\begin{equation}\label{formula:initial condition of DP}
dp[0, p_0] = (0, 0)
\end{equation}

\begin{lemma}
    The number of observation points for $n$ objects is $O\left(\frac{n^3{\epsilon}+n^2}{\epsilon^3}\right)$.
\end{lemma}
\begin{proof}
    Provided by Theorem 4.4 in~\cite{deploy}, the number of observation points for $n$ objects is $O\left(\frac{n^2}{\epsilon^3}\right)$. By using the discretization method of $\delta$, the discretization interval is $\frac{{\epsilon}D}{n}$. In this case, there are at most $\frac{d_{max}n}{{\epsilon}D}$ points in the distance domain and at most $\frac{d_{max} {\theta}  n}{{\epsilon}D}$ points in the angle domain, the multiple of the which is the maximum number of observation points for one object. Therefore, the number of all the observation points $|\mathcal{P}|$ is:
    \begin{align}
        &O\left(\frac{n^2}{\epsilon^3}\right)+O\left(\frac{nd_{max}}{\epsilon{D}} \cdot \frac{nd_{max}{\theta}}{\epsilon{D}} \cdot n\right)
        =O\left(\frac{n^2}{\epsilon^3}+\frac{n^3}{\epsilon^2}\right)
    \end{align}
\end{proof}

\begin{theorem} \label{theorem}
    The dynamic programming algorithm runs in $O\left((\frac{n^2+n^3\epsilon}{\epsilon^3})^2 \cdot \log(\frac{n^2}{\epsilon})\right)$ time and compute a $(1+\epsilon)$-approximation solution on both path length and observation quality.
\end{theorem}

\begin{proof}[Proof of Theorem 1]
    The size of the dynamic programming table is $O\left(n \cdot |\mathcal{P}|\right)$. However, only the observation points of the current object are considered when filling out the table\footnote{Without loss of generality, we assume that all objects have the same number of observation points.}, hence the outer loop requires time of $O\left(n \cdot \frac{|\mathcal{P}|}{n}\right)$. The inner loop is to traverse over the previous part of the table with the same complexity of $O\left(n \cdot \frac{|\mathcal{P}|}{n}\right)$, and then to update the set of length and observation quality. A new pair is added to the set at each update, which can be maintained as ordered in $O\left(\log |L_s|\right)$ time. Now we focus on calculating $|L_s|$, which is the maximum number of pairs in the set. 

    The maximum path length in this problem can be bounded by $O\left(n(D+2d_{max})\right)$ and using interval $\delta$ ensures that the previous discretization method is PTAS. In this case, all the possible distance value in the set is rounded to an integer multiple of $\delta$, thus $|L_s| \leq \frac{n(D+2d_{max})}{\delta}$. Consequently, the time complexity of the domination operation on the set is: 
    \vspace{1em}
    \begin{align}
        O\left(\log(\frac{n(D+2d_{max})}{\delta})\right)
        =&O\left(\log(\frac{n^2(D+2d_{max})}{{\epsilon}D})\right)
        =&O\left(\log(\frac{n^2}{\epsilon})\right)
    \end{align}
    
    Integrating the previous outcomes, we get the overall time complexity as:

    \begin{align}
        O\left(n \cdot\frac{|\mathcal{P}|}{n}\cdot n \cdot \frac{|\mathcal{P}|}{n} \cdot \log(\frac{n^2}{\epsilon})\right)
        =O\left((\frac{n^2+n^3\epsilon}{\epsilon^3})^2 \cdot \log(\frac{n^2}{\epsilon})\right)
    \end{align}
\end{proof}

\subsection{Determination of observing order}

Any feasible route contains a certain observing order of the objects. Even though the dynamic programming algorithm above accepts arbitrary observing orders, it is still demanding in searching for better orders that lead to lower path costs. Since it is impractical to enumerate all the $\frac{A^n_n}{2}$ possible orders, in this work, several heuristic algorithms are proposed to search the orders, which are compared later in the experiments to evaluate the performance.

\subsubsection{Random Selection (RS)} \label{sec:random select}
The first method is randomly selecting one observation point for each object. After obtaining all the points to visit, a classic TSP~\cite{1.5TSP} algorithm is used to solve the path on the starting point and the selected observation points.

\subsubsection{Nearest Point First (NPF)} \label{sec:nearest}
The second method (shown in Algorithm~\ref{algo:Nearest point first}) is setting off from the starting point and always choosing the nearest observation point that belongs to an object not visited yet. The process is continuously repeated until all objects are observed. The UAV then directly returns to the starting point.

\begin{algorithm}[ht]
  \caption{Nearest Point First} 
    \label{algo:Nearest point first}
    \begin{algorithmic}[1]
      \Require
        Object set $\mathcal{O}$, Observation point set $\mathcal{P}$, starting point $P_0$
      \Ensure
        $path$: The visiting order of the route
        \State $path \gets {P_0} $
        \While{$\mathcal{O}$ not empty}
            \State choose the nearest point $p_{next}$ from $\mathcal{P}$
            \State current location $\gets p_{next}$
            \State $ path.append(p_{next})$
            \State delete objects O' from $\mathcal{O}$ that can be observed at $p_{next}$
            \State delete observation points of O' from $\mathcal{P}$
        \EndWhile
    \end{algorithmic}
\end{algorithm}

\subsubsection{Generalized TSP (GTSP)} \label{sec:GTSP}
Generalized TSP is a variant of TSP that chooses at least one node from each zone and solves a path on these nodes. This part uses the same 
clustering method as ILP. After that, a Generalized TSP solver~\cite{GTSP} is carried out to solve a path (order) on these clusters.

\subsubsection{TSP over Objects (TSPO)} \label{sec:TSP on object}
This method is relatively intuitive, which directly views all the objects and $P_0$ as vertices and then uses the algorithm from~\cite{1.5TSP} to get the vising order. 

\subsubsection{TSP on Lower Bound (LBTSP)} \label{sec:TSP on LB}
Given the proposed lower bound in Section~\ref{sec: lower bound}, there is a weighted graph essentially including all the objects and the starting point $P_0$. Even though the graph may not be subject to the Euclidean metrics, we can still obtain a visiting order by conducting the algorithm from~\cite{1.5TSP} on it. 

\subsection{Lower bound} \label{sec: lower bound}
To calculate the approximation ratio of the algorithms in experiments, a lower bound of the problem is designed and is described in this part.(shown in Algorithm~\ref{algo:lower bound})

We consider the observation points of each object to be one cluster and $P_0$ as an independent cluster, similar to the design of ILP in Section~\ref{sec:ILP}. Each cluster is then viewed as a single vertex after solving the minimum distance between each of them. If there is an overlapping area between two clusters, the minimum distance between them is 0. This distance is used as the weight of the edges among vertices, hence a graph is formed. The lower bound is exactly the cost of the minimum spanning tree of the graph. 

\begin{algorithm}[ht]
  \caption{Compute a lower bound} 
    \label{algo:lower bound}
    \begin{algorithmic}[1]
      \Require
        Object set $\mathcal{O}$, Observation point set $\mathcal{P}$, and starting point $P_0$
      \Ensure
        $LB$: cost of the lower bound
        \State consider $P_0$ as $cluster_0$
        \For{$o_i \in O$}
            \State consider all the observation points that can observe $o_i$ as $clutser_i$
        \EndFor
        \State create a graph $G$ that treat each $cluster_i$ as a vertex
        \State compute the minimum distance between every two clusters and set it as the weighted edge
        \State $LB \gets$ cost of minimum spanning tree of $G$
    \end{algorithmic}
\end{algorithm}

\section{Experiments}
\label{sec: experiments}
In this section, we evaluate the performance of the proposed algorithms through a series of simulated experiments. The experiments contain two parts, the first of which focuses on the numerical results of synthetic cases while the second is conducted in the Airsim simulator~\cite{Airsim} to test the performance in the realistic virtual environment.

\subsection{Experiments on synthetic cases}
A number of objects are placed randomly on a map with positive coordinates and orientations of 0 to 360 degrees. The map size is set to 200 meters. To ensure both the quality and safety of the observation, we set the minimum observing distance to 2 m. The maximum observing angle of an object, as defined $\theta$, is set to 30 degrees~\cite{deploy} and $\epsilon$ is 0.5.

\subsubsection{Brute force experiments} \label{sec:brute_force_experi}
Brute force is to find the optimal observing order, based on which the dynamic programming yields the lowest path length for a fixed quality requirement. The number of objects $n$ is set to be $\left\{3, 4, 5, 6, 7, 8\right\}$ and observing distance 10 m with 250 random cases for each. The quality requirement is calculated by multiplying the highest quality $n*q_{max}$ with $q^*$ where $q^* \in \left\{0.3, 0.4, 0.5, 0.6, 0.7, 0.8, 0.9\right\}$. After obtaining the optimal path, we further run the five order determination algorithms before dynamic programming, the output of which are compared with the optimal results from the brute force. 

\subsubsection{Different range and angle experiments} \label{sec:different_range_angle_experi}
In the second part, $n$ is set to $\left\{5, 10, 15, 20, 25, 30\right\}$ and for each $n$ we further set the maximum observation distance $d_{max}$ to $\left\{4, 6, 8, 10, 12\right\}$ to analyze the difference in the approximation ratio. For each $n$, as in the first part, we randomly generate 200 cases and change $d_{max}$ for each case. The order determination and dynamic programming algorithms are then applied on these cases and results with different $q^*$ are recorded. 

All the experiments are conducted on the Linux server, with Intel (R) Xeon (R) Gold 5215 CPU 2.50GHz and 188 GB RAM.

\subsection{Experiments in the simulator}
\label{sec: simluation}
In addition to the tests on synthetic data, we carried out our algorithms in the Airsim simulator for practical applications, which is an open-source simulator developed by Microsoft for training and testing autonomous vehicles and drones in realistic virtual environments. The simulation consists of two parts. Firstly, We create 50 cylinders as objects in the 3DsMax software with radius of 1 m and height of 2 m. Subsequently, we paste on the cylinders high definition English advertisement images downloaded from the Internet, which are then imported into the simulator (see Figure~\ref{fig:recognized_imgs}). The map size is set to 150 m and 30 cases are generated for each $n$ in $\left\{5, 10, 15, 20\right\}$. For each case, $n$ objects are randomly chosen from the 50 ones and $q^*$ is set to $\left\{0.3, 0.5, 0.7, 0.9\right\}$. The drone will fly along the paths produced by different order determination algorithms and dynamic programming algorithm. Additionally, the initial paths are included provided by RS, NPF, and GTSP without adjustment. We also add a path that yields the highest observation quality from all the objects. The path is solved by using the algorithm from~\cite{1.5TSP} on the observation points that produces maximum observation quality of each object. During the flight, the UAV takes photos of the objects which are sent to the CnOCR API in Python~\cite{cnocr}. The observation quality is reflected in the accuracy of the recognized words in the photos.

\begin{figure*}[ht]
\centering
\begin{minipage}[b]{1\linewidth}
    \begin{subfigure}{0.32\textwidth}
        \centering
        \includegraphics[width=\linewidth]{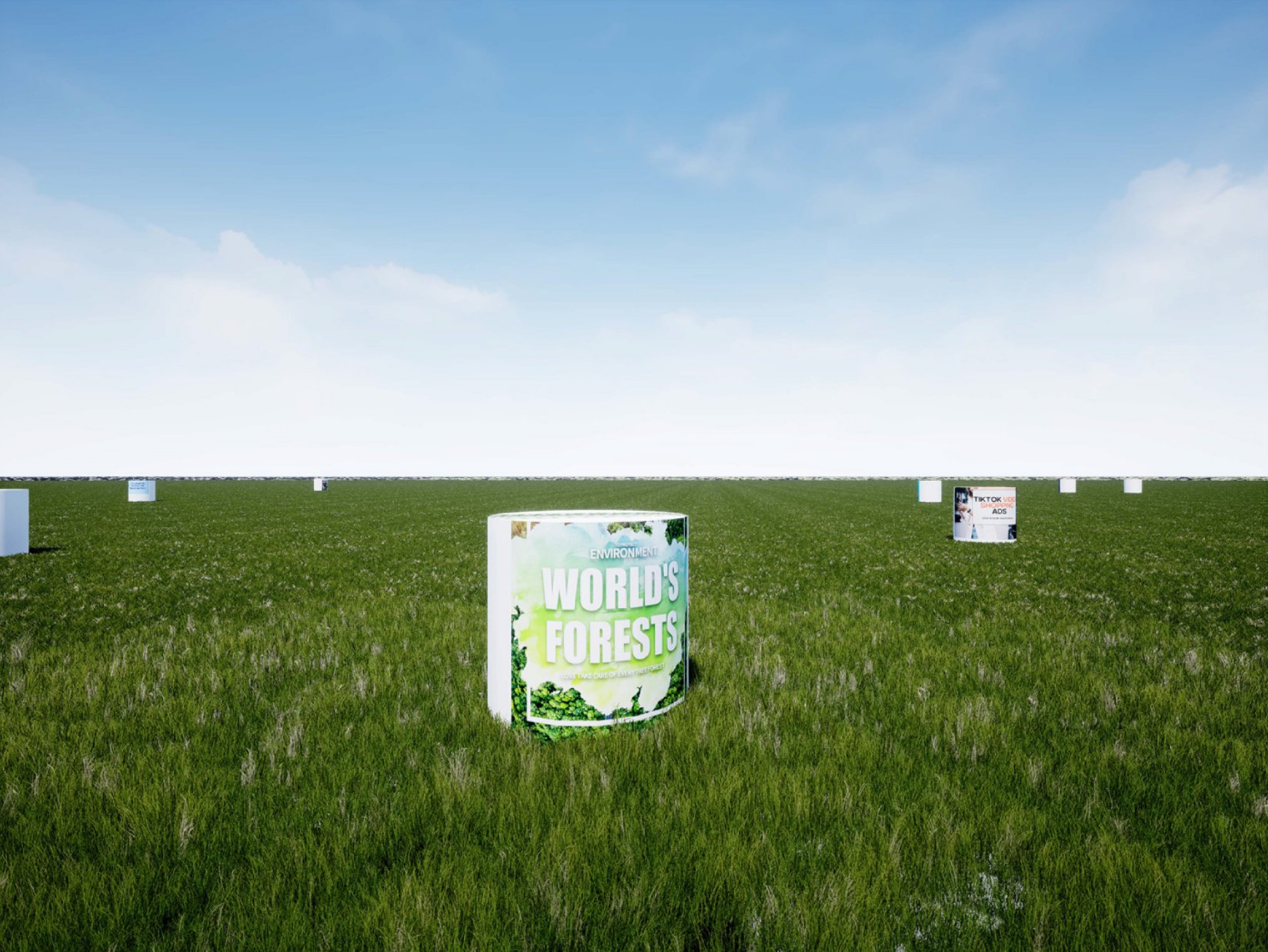}
        \captionsetup{justification=centering}
        \caption{'WORLDSFORESTS' \\ 13 / 13}
        \label{fig:recognized_img1}
    \end{subfigure}
    \begin{subfigure}{0.32\textwidth}
        \centering
        \includegraphics[width=\linewidth]{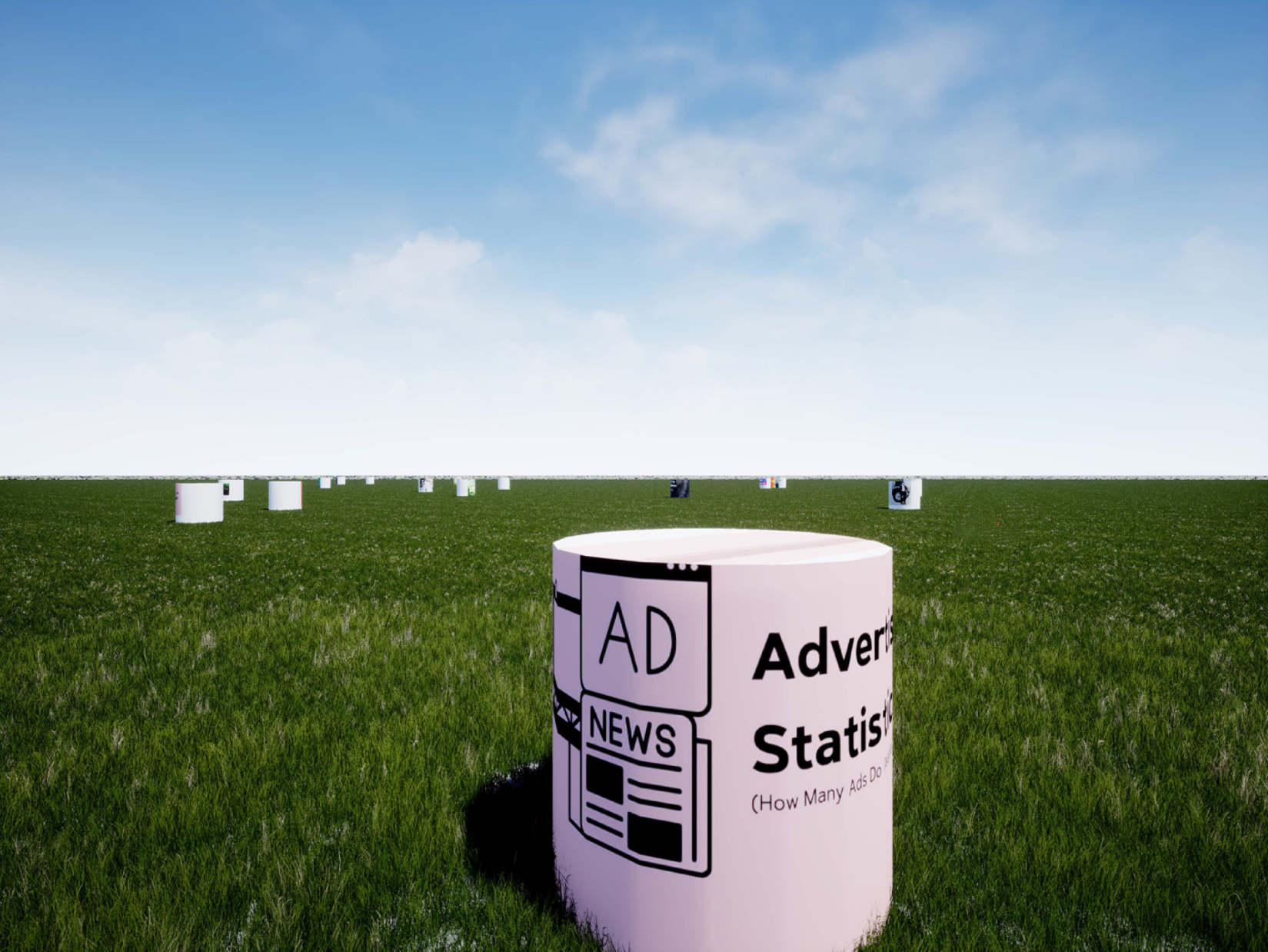}
        \captionsetup{justification=centering}
        \caption{'AdveStatis' \\ 10 / 21}
        \label{fig:recognized_img2}
    \end{subfigure}
    \begin{subfigure}{0.32\textwidth}
        \centering
        \includegraphics[width=\linewidth]{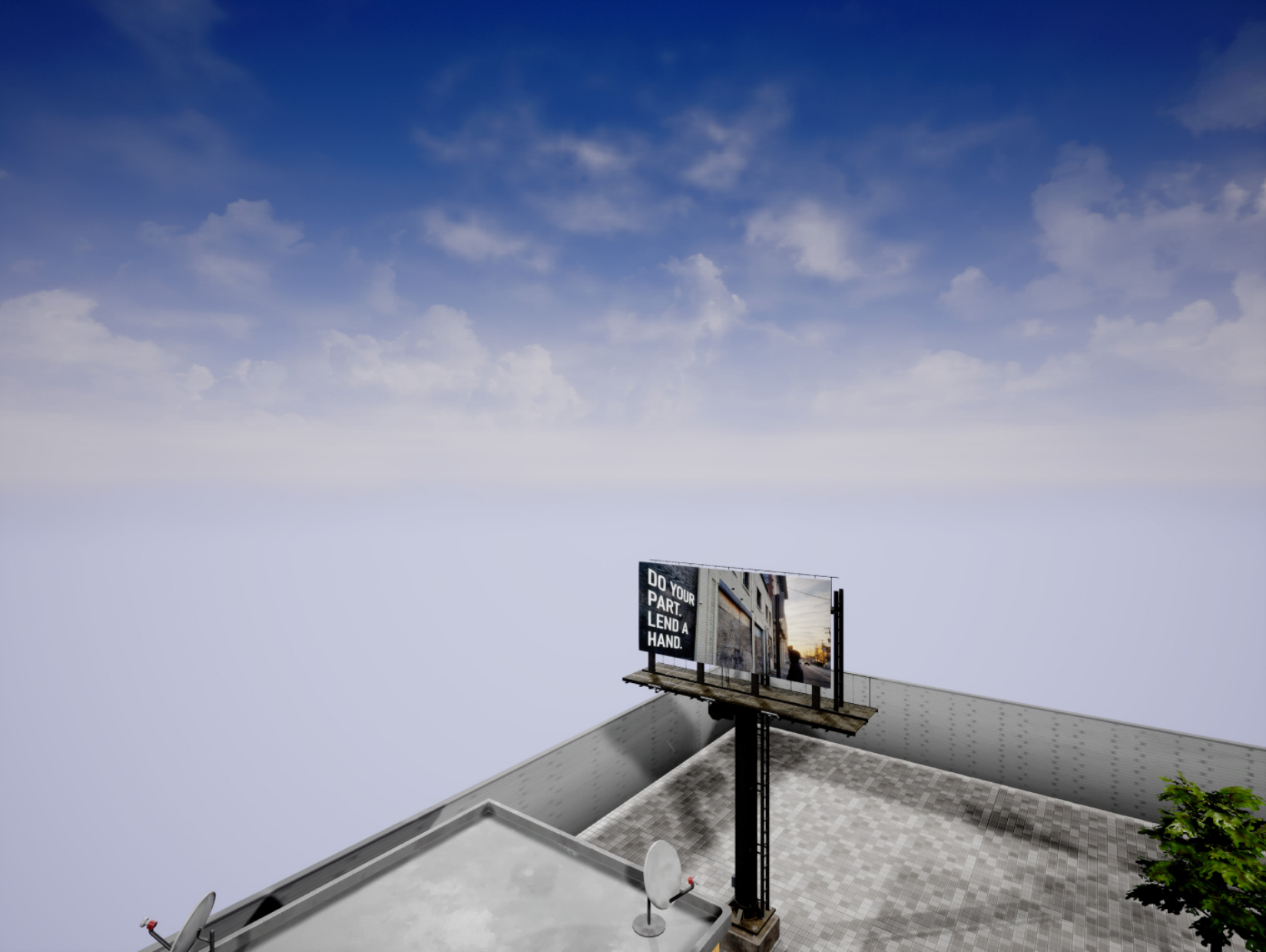}
        \captionsetup{justification=centering}
        \caption{'DLENDHAND' \\ 9 / 21}
        \label{fig:recognized_img3}
    \end{subfigure}
\end{minipage}
\caption{Photos taken by UAV of observed objects and the corresponding recognition results. For each subplot, the recognized string is presented before showing the number of correct letters and the label length.}
\label{fig:recognized_imgs}
\end{figure*}

Furthermore, a virtual city environment was constructed by using the 'Stylized Town Package' and 'Billboards VOL.1' packages from Unreal Engine Market~\cite{unreal-engine}. Dozens of different billboards and buildings with boards of names were randomly placed in the map to simulate a real and complex city (see Figure~\ref{fig:virtualcity}). The performance of the algorithms were then validated in this environment by setting $n$ to 15 and treating the boards of the buildings and billboards as objects. The flight height is slightly higher than the tallest building to avoid collision. We ran 15 cases in total, with the concrete processes and goals same as the first part. 

\begin{figure}[hbt]
\centering
\includegraphics[width=0.65\textwidth]{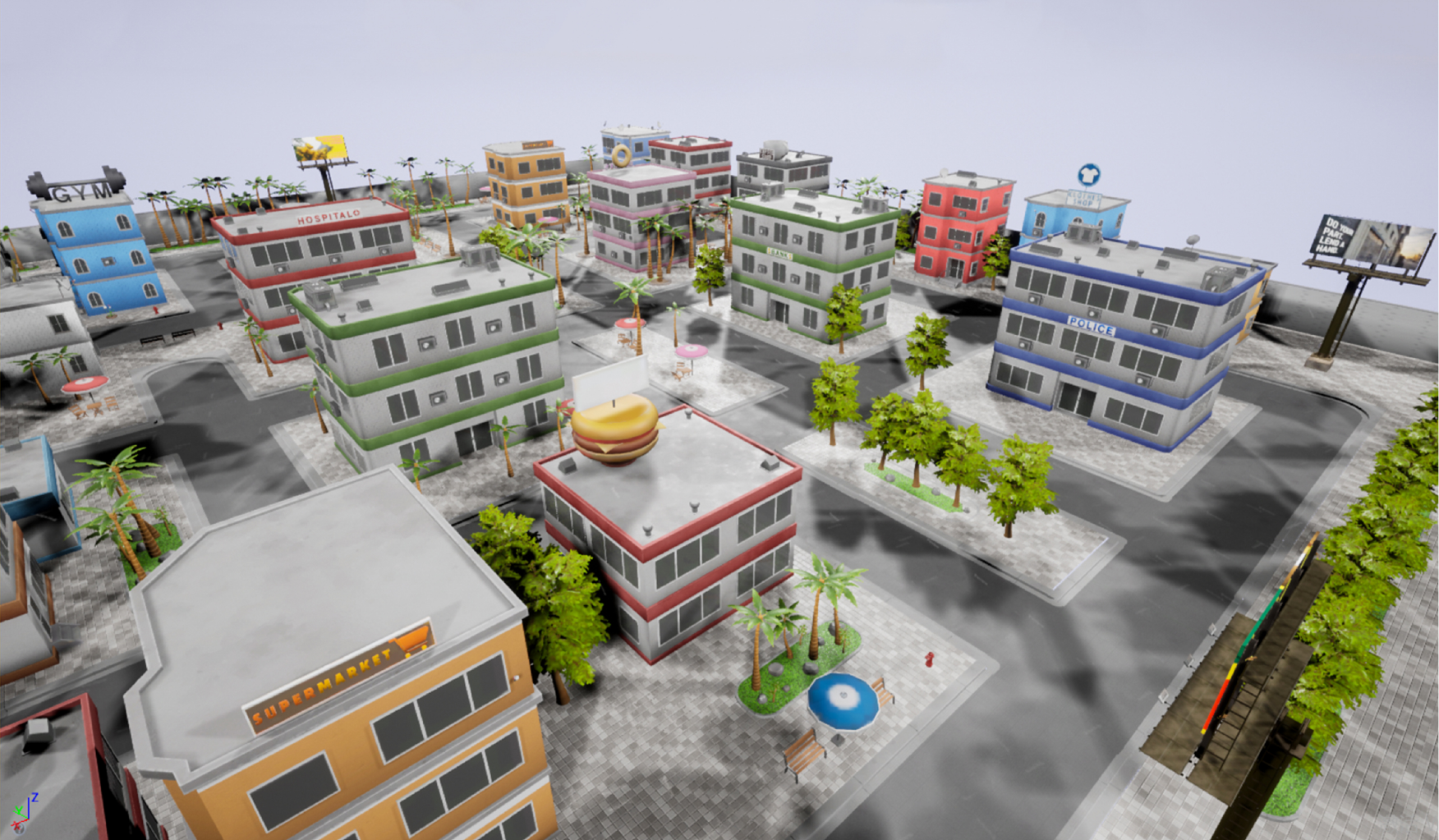}
\captionsetup{justification=centering}
\caption{Virtual city environment in the simulator}
\label{fig:virtualcity}
\end{figure}

To calculate the recognition accuracy in one case, we solve the longest common subsequence between the recognition result and the label of the object as the correctly recognized string. Suppose that the labels of the objects in one case are $\mathcal{L} = \left\{l_1, l_2, ..., l_m\right\}$, the correctly recognized strings $\mathcal{C} = \left\{c_1, c_2, ..., c_m\right\}$. We define the accuracy as the sum of all correctly recognized letters over all the letters in the labels, $\frac{\sum_{i} {|c_i|}}{\sum_{i} {|l_i|}}$. Figure~\ref{fig:recognized_imgs} shows the photos taken during the flight and the according OCR results.

Experiments on synthetic cases consider that the observation distance and angle are equal for all objects. In the simulation trails, however, owing to the different size of the images and words, the observing distance and observing angle are calculated separately for each object. We stipulate that the UAV should recognize at least $30\%$ correct letters for each object from some observation point. Based on this, $d_{max}$ and $d_{min}$ are searched via a binary search algorithm, after placing the UAV right on the front of the object. The observing angle $\theta$ is then obtained from a second binary search algorithm that guides the UAV to fly on the arc with the radius of $d_{max}$. Figure~\ref{fig:seefromtop} is an overview perspective showing the different observation range of objects and Table~\ref{table:objects_info} contains a proportion of the objects' information.

\begin{figure}[ht]
\centering
\includegraphics[width=0.65\textwidth]{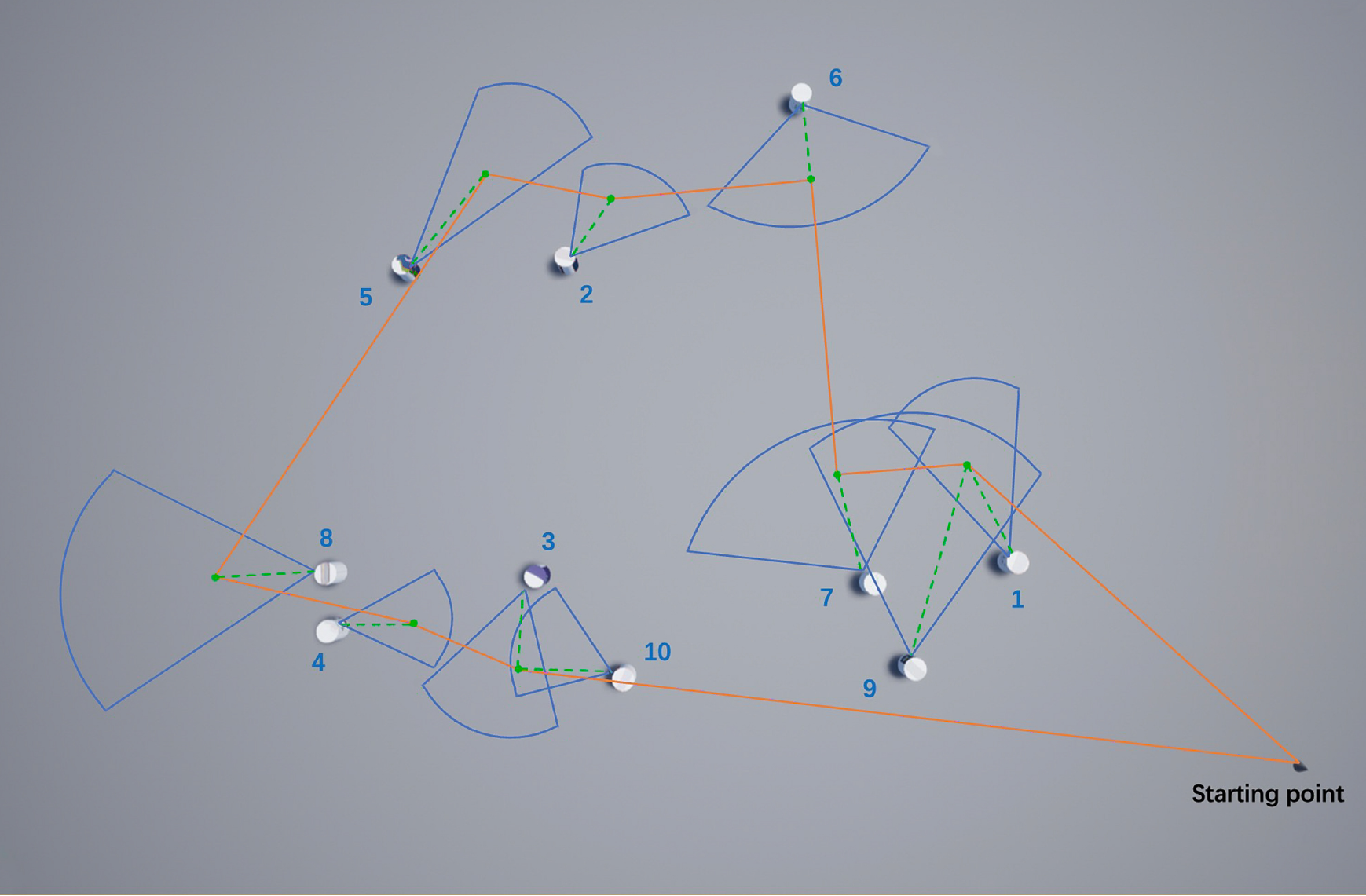}
\captionsetup{justification=centering}
\caption{An overview of the objects' observation range and a proposed path}\label{fig:seefromtop}
\end{figure}

\begin{table}[htbp]
\centering
\caption{Labels, maximum and minimum observation distance and maximum observation angles of some objects}\label{table:objects_info}
\small
\begin{tabular}{cccc}
\hline
object label & $d_{max}$  & $d_{min}$ & $\theta$  \\ \hline
Invincible & 23 & 5 & 26 \\ 
Pageantry & 29 & 4 & 33 \\ 
English Olympiad & 28 & 5 & 25 \\ 
Facebook Ads preparation & 21 & 5 & 23 \\ 
Account Based Marketing & 19 & 6 & 30 \\ 
PLEASE KEEP OFF THE GRASS & 16 & 5 & 33 \\ 
NOW WHICH SNACK PROTEIN PACK & 28 & 3 & 28 \\ 
TAKE CONTROL OF YOUR HEALTH & 26 & 4 & 14 \\ 
AXE STYLING EASY FOR EVERYDAY & 12 & 4 & 30 \\ 
How to Get Facebook Ads Working For You in 2023  & 13 & 4 & 28 \\ \hline
\end{tabular}
\end{table}

\subsection{Numerical results}
\label{sec:numerical_results}

In this section, the statistic results are presented in the following figures and tables. We first present the results of the brute force experiments, which are the comparison of our order determination algorithms with the optimal order. The results of different range and angle experiments are then demonstrated before presenting the running time. Finally, we show the simulation results.

\subsubsection{Comparison with brute force} \label{sec:brute_force_result}

Based the statistics, the approximation ratio remains less than 1.12 and is basically stable for every $n$ and $q^*$. There is a slight rise when $q^*$ is greater than 0.8. In every situation, GTSP algorithm always yields the best results and and LBTSP is the second only to it, with TSPO and RS showing fairly similar performance. NPF algorithm produces the worst results when integrated with the dynamic programming. On the whole, our algorithm yields approximation ratio very close to 1 and GTSP and LBTSP are the two best order searching methods.

\subsubsection{Different $n$ and range} \label{sec:different n and range}
We fix the observation range to 4, 8 and 12, and plot Figure~\ref{fig:effect_of_n}, in which the effect of $n$ is illustrated. After that, $n$ is fixed and the effect of observation range is shown in Figure~\ref{fig:effect_of_range}.

\begin{figure*}[ht]
\centering
\begin{minipage}[b]{1\linewidth}
    \begin{subfigure}{0.32\textwidth}
        \centering
        \includegraphics[width=\linewidth]{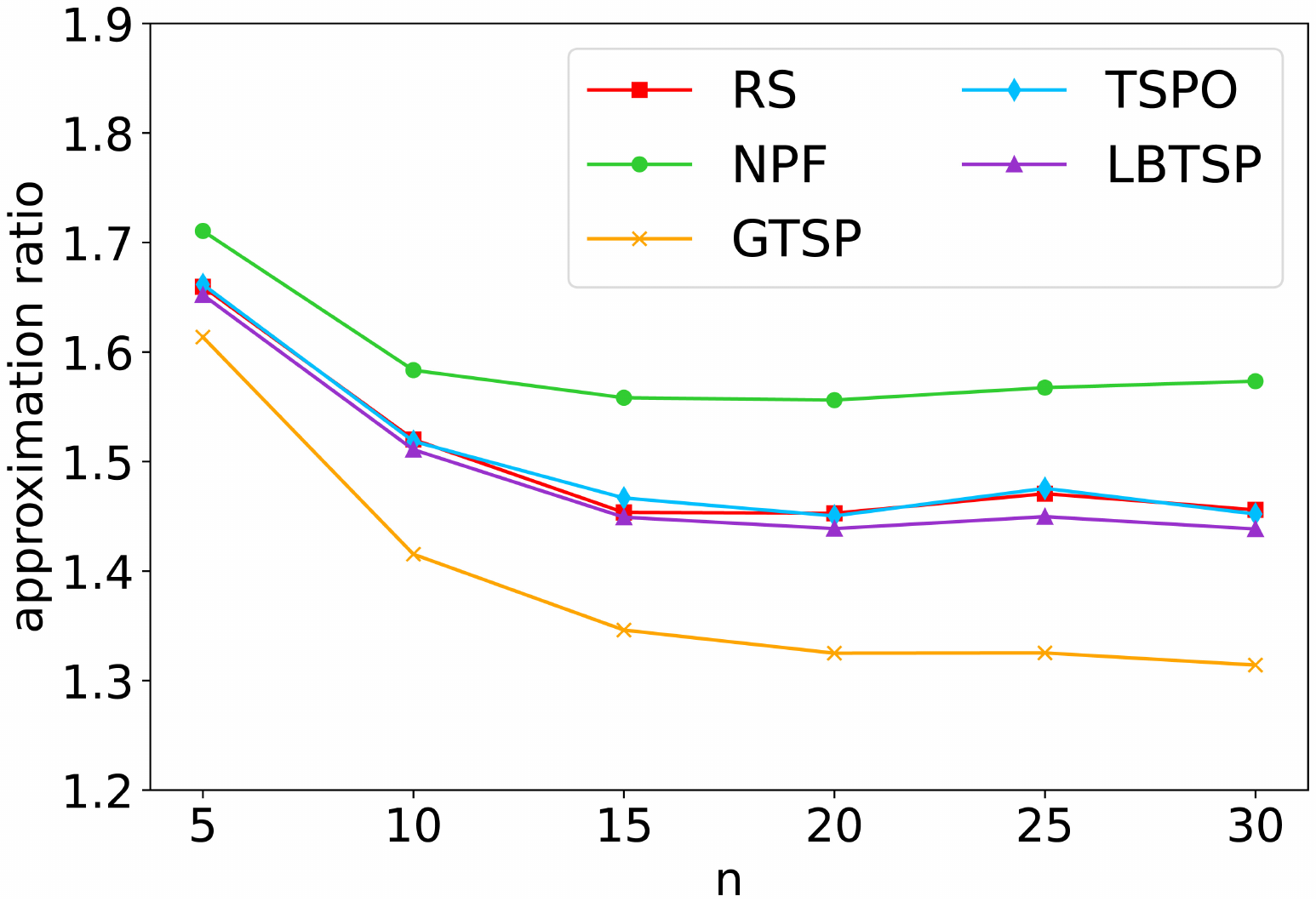}
        \captionsetup{justification=centering}
        \caption{observation range = 4}
        \label{fig:eff_n_4}
    \end{subfigure}
    \begin{subfigure}{0.32\textwidth}
        \centering
        \includegraphics[width=\linewidth]{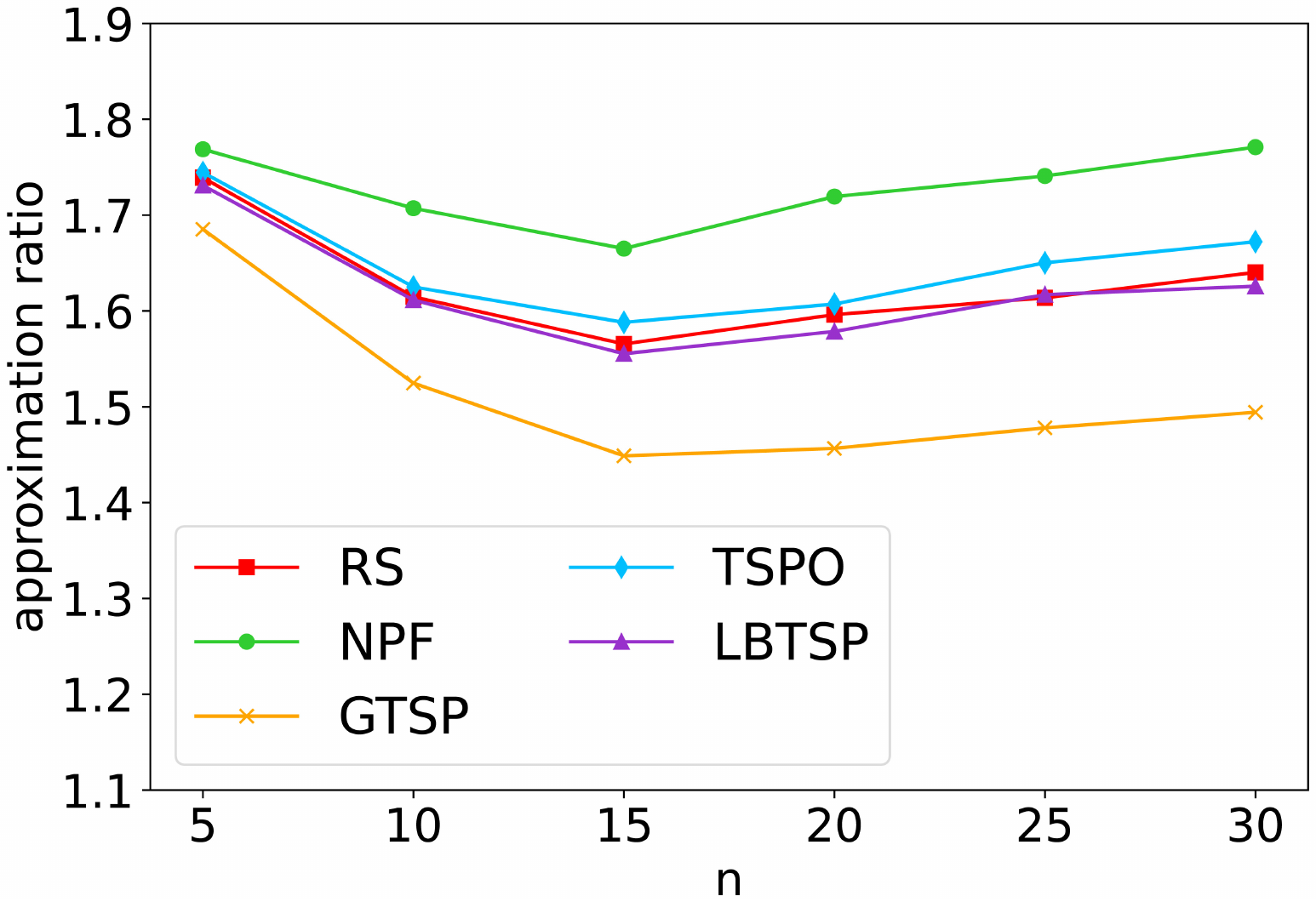}
        \captionsetup{justification=centering}
        \caption{observation range = 8}
        \label{fig:eff_n_8}
    \end{subfigure}
    \begin{subfigure}{0.33\textwidth}
        \centering
        \includegraphics[width=\linewidth]{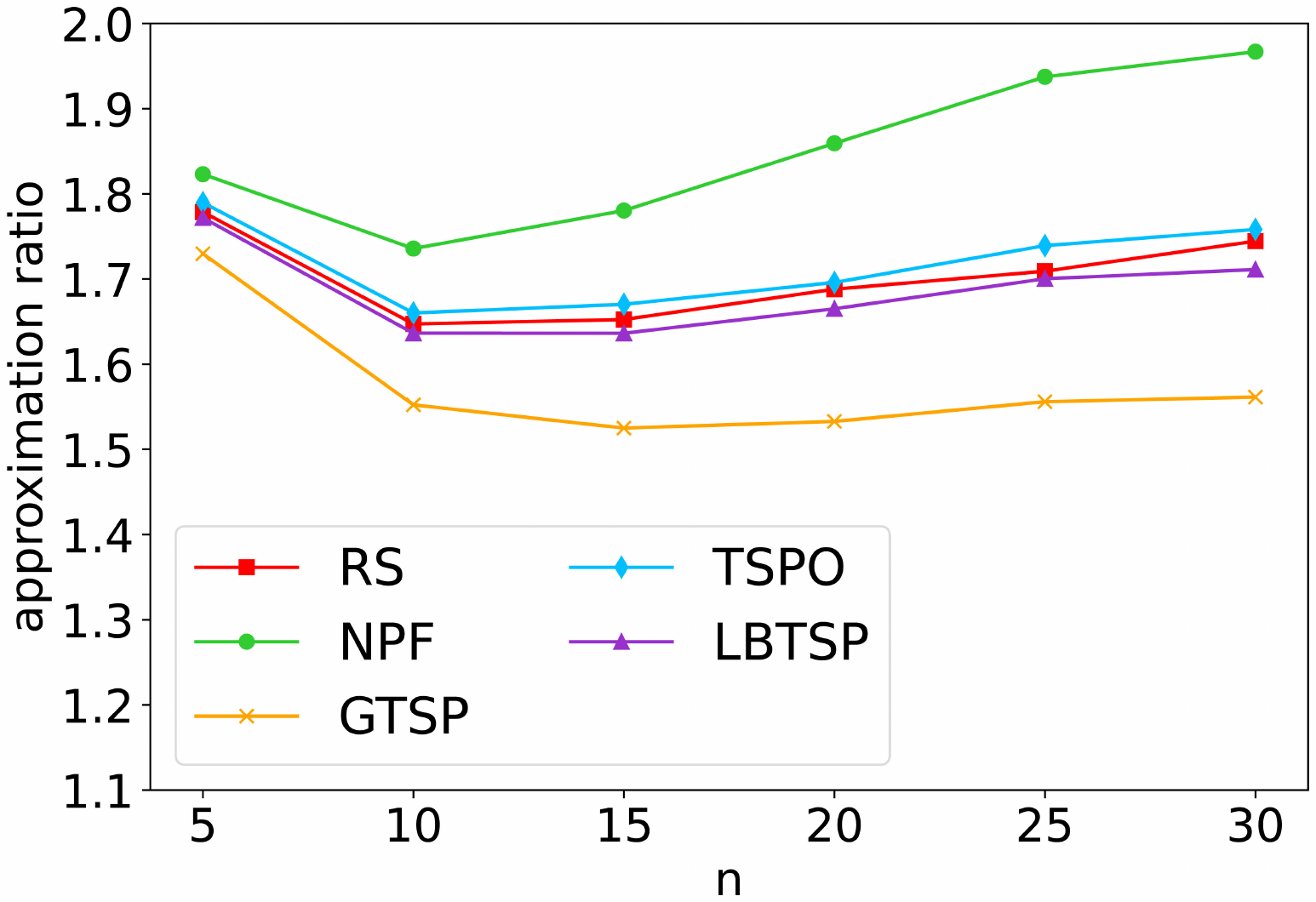}
        \captionsetup{justification=centering}
        \caption{observation range = 12}
        \label{fig:eff_n_12}
    \end{subfigure}
\end{minipage}
\captionsetup{justification=centering}
\caption{Effect of object number with different observation range}
\label{fig:effect_of_n}
\end{figure*}
\vspace{1.5em}

\begin{figure*}[ht]
\centering
\begin{minipage}[b]{1\linewidth}
    \begin{subfigure}{0.32\textwidth}
        \centering
        \includegraphics[width=\linewidth]{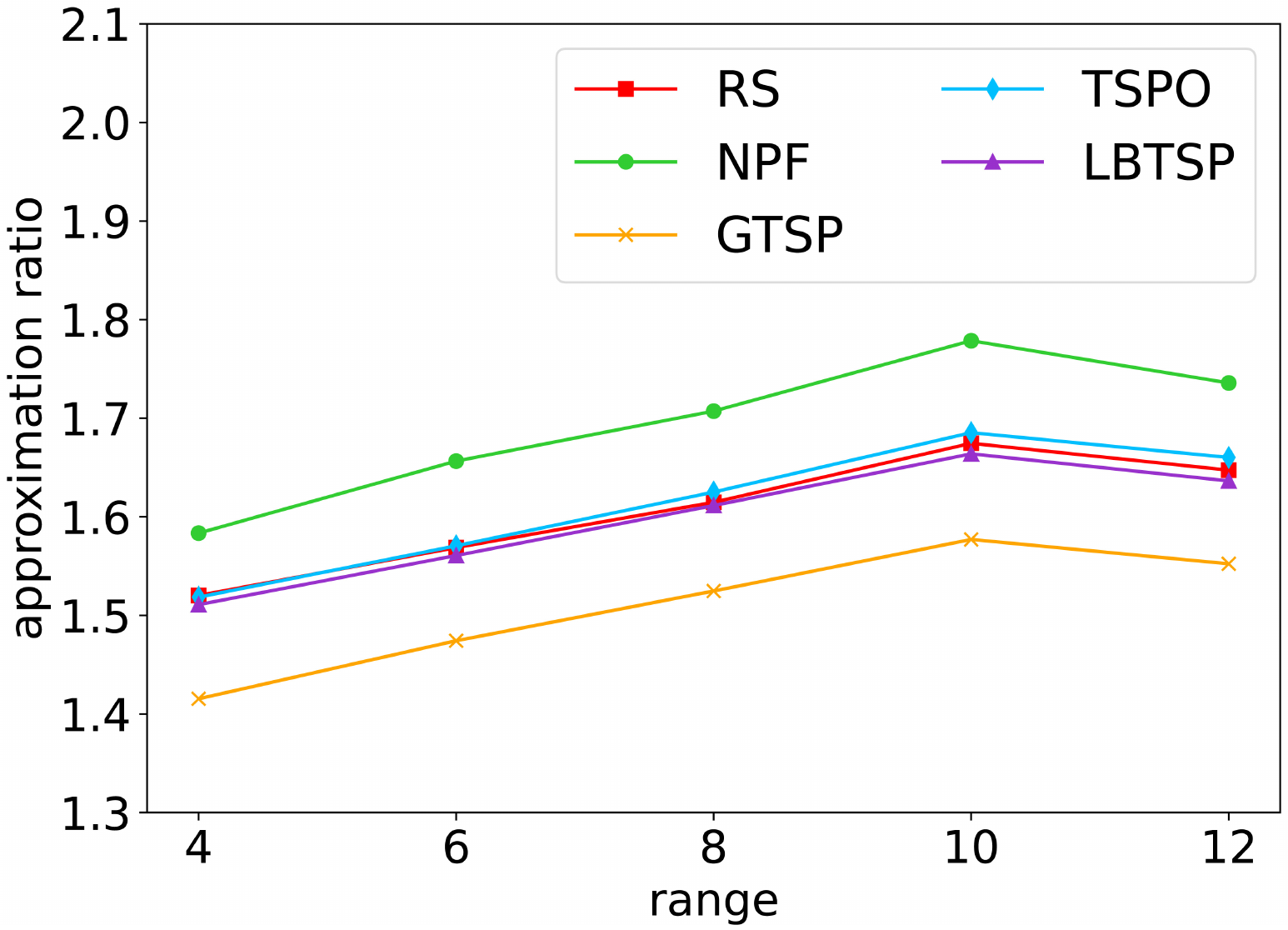}
        \captionsetup{justification=centering}
        \caption{object number = 10}
        \label{fig:eff_range_10}
    \end{subfigure}
    \begin{subfigure}{0.32\textwidth}
        \centering
        \includegraphics[width=\linewidth]{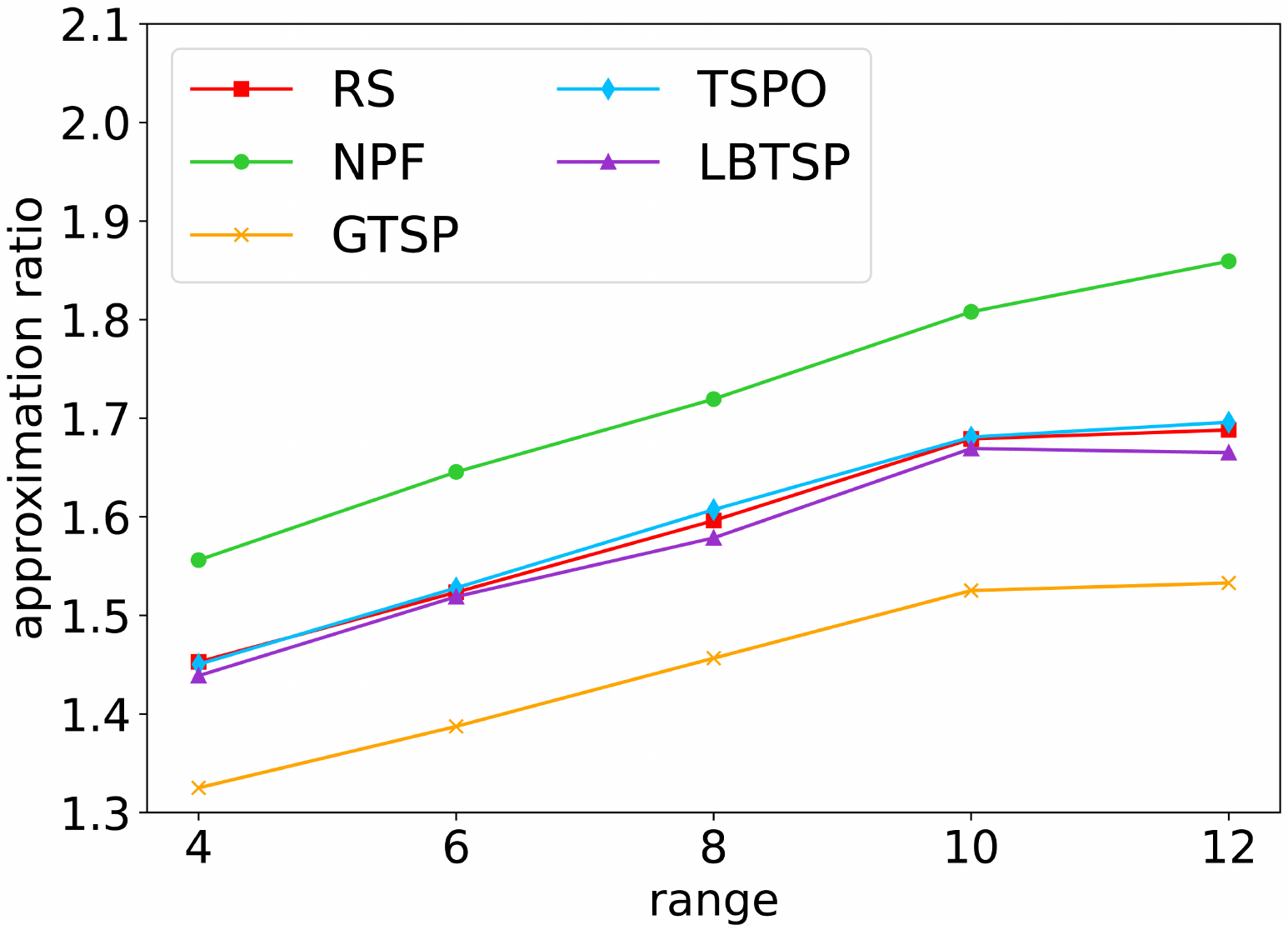}
        \captionsetup{justification=centering}
        \caption{object number = 20}
        \label{fig:eff_range_20}
    \end{subfigure}
    \begin{subfigure}{0.33\textwidth}
        \centering
        \includegraphics[width=\linewidth]{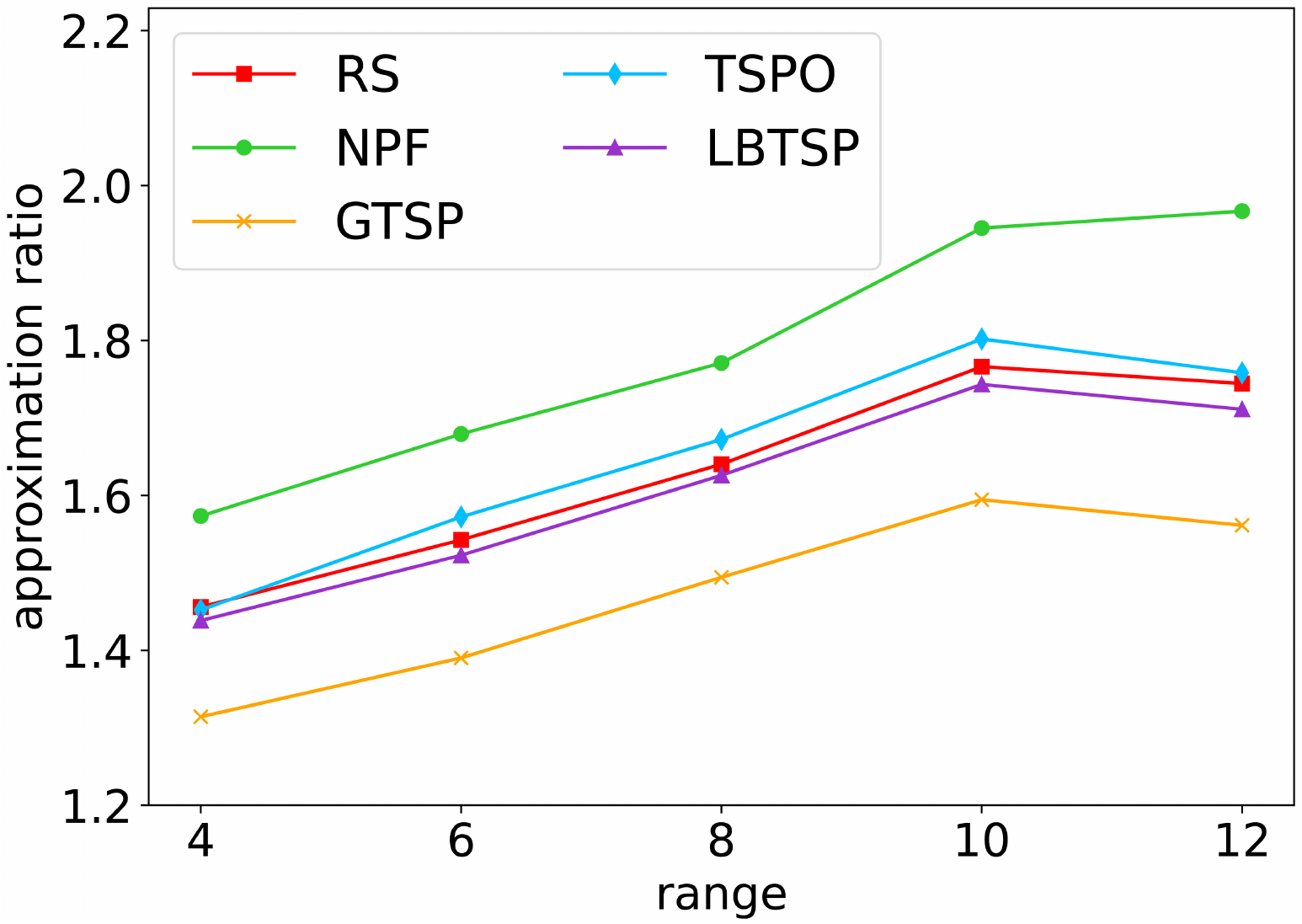}
        \captionsetup{justification=centering}
        \caption{object number = 30}
        \label{fig:eff_range_30}
    \end{subfigure}
\end{minipage}
\captionsetup{justification=centering}
\caption{Effect of observation range with different object number}
\label{fig:effect_of_range}
\end{figure*}
When $d_{max} = 4$, the overall approximation ratio has a declining tendency with 1.61 when $n$ is 5 and 1.31 when $n$ is 30 for the GTSP method. Other order determination methods produce similar results, and the relative superiority is shown to be the same as that in the brute force section. NPF achieves 1.71 approximation ratio when $n$ is 5 which gradually declines to 1.58 when $n$ is 30. The results bear some differences for larger observation ranges. When the range is 8 and 12, the approximation ratios both undergo a decline when $n$ is smaller than 15 and then experience a rebound. Concretely, for cases where the range equals 8, The GTSP methods achieves the approximation ratio of 1.69 when $n$ is 5, which gradually descends to 1.45 when $n$ is 15. The ratio then rebounds to 1.49. When $d_{max}$ is 12, the approximation ratio ranges from 1.51 to 1.98, slightly higher than the overall situation in $d_{max}=8$. The performance of NPF method seems to be susceptible to the growing $n$ but the other method remains relatively stable. In another word, our algorithms are robust to the growing number of objects.

We then analyse the effect of observation range. The number of objects is selected to be respectively 10, 20, and 30. It is apparent to notice the growing trend in approximation ratio with the rise in the observation range from 4 to 10 for all $n$. The ratio then remains stable when the range further increases to 12. When $n$ is 10, GTSP method achieves the approximation ratio ranging from 1.41 to 1.59, roughly 0.1 or 0.2 smaller than the other methods. The results bear a resemblance for large $n$, in which the ratio increases from 1.32 to 1.52 when $n$ is 20 and 1.32 to 1.6 when $n$ is 30. As a brief summary, there is a modest increase in approximation ratio along with the observation range and number of objects, in that the algorithms probably have slight difficulty in solving complex cases.

\subsubsection{Effectiveness of dynamic programming algorithm}
According to the experimental results of the previous two sections, it can be seen that the dynamic programming algorithm can achieve optimal performance combined with the GTSP observation order. This part of the experiment compares our algorithm with 4 baseline algorithms. The baseline methods are: Random Selection (RS), Nearest Point First (NPF)~\cite{wtsc}, Generalized Traveling Salesman Problem (GTSP)~\cite{GTSP}, and Maximum Observation Quality Path (MaxQ). Since the TSPO and LBTSP methods only generate the order of observed objects rather than a feasible path, they are not used as baseline algorithms for comparison in this part.

We have counted the percentage of cases in which RS, NPF, and GTSP methods can achieve the quality requirement without the assistance from our dynamic programming algorithm, the effectiveness and the necessity of which is to be explored. TSPO and LBTSP methods are omitted since they only yield the visiting order instead of feasible paths. For those cases in which the order determination methods have already been capable of achieving the quality requirement, we calculate the percentage by which the dynamic programming algorithm is able to reduce the path lengths.

The number of objects is selected to be 5, 15, and 25. As the quality requirement grows, the general tendency is extremely in accordance with the expectation. Concretely, the satisfying percent of all the three methods remains high when $q^* \leq 0.5$, with GTSP order performing slightly poorly than the others. For stricter $q^*$, however, the percentage drops rapidly to $20\%$ and finally to 0, which is given by NPF method when $q^*$ is 0.9. Based on the outcomes above that the order determination algorithms can barely find satisfying paths for large $q^*$, it is necessary to use dynamic programming to improve the path in order to fulfill the quality requirement. 

As shown in Figure~\ref{fig:path_length_of_algorithms}, our algorithm solves a shorter path than the baselines. Since the baseline algorithms did not consider observation quality, the paths are presented as straight lines in the figure. Compared to other algorithms, the algorithm proposed in this study can obtain significantly shorter paths. For example, when $n$ is 5, the average path length of the baseline algorithm RS is about 521m, NPF is 524m, and GTSP is about 509m, while the average path length of the algorithm proposed in this study is only 491m, and it can also meet the perceptual quality constraint. When seeking relatively higher perceptual quality, there is a trade-off between perceptual quality and path length. Compared to the MaxQ path in the figure, the algorithm in this study can search for a path with $90\%$ of the maximum perceptual quality but with a significantly shorter path length, reducing the original path by about 79m.
\begin{figure*}[ht]
\centering
\begin{minipage}[b]{1\linewidth}
    \begin{subfigure}{0.32\textwidth}
        \centering
        \includegraphics[width=\linewidth]{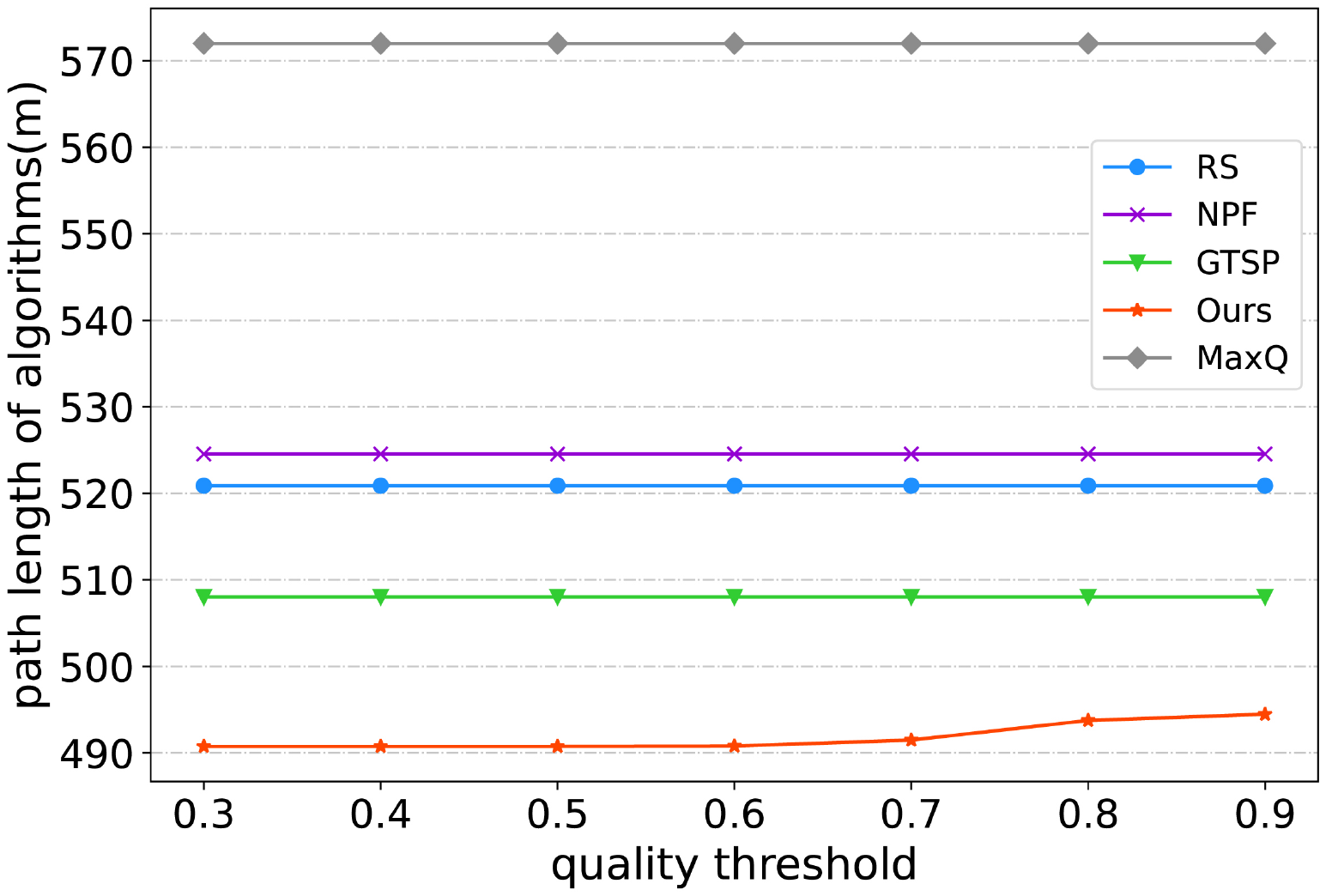}
        \captionsetup{justification=centering}
        \caption{object number = 5}
        \label{fig:pathlength-5}
    \end{subfigure}
    \begin{subfigure}{0.32\textwidth}
        \centering
        \includegraphics[width=\linewidth]{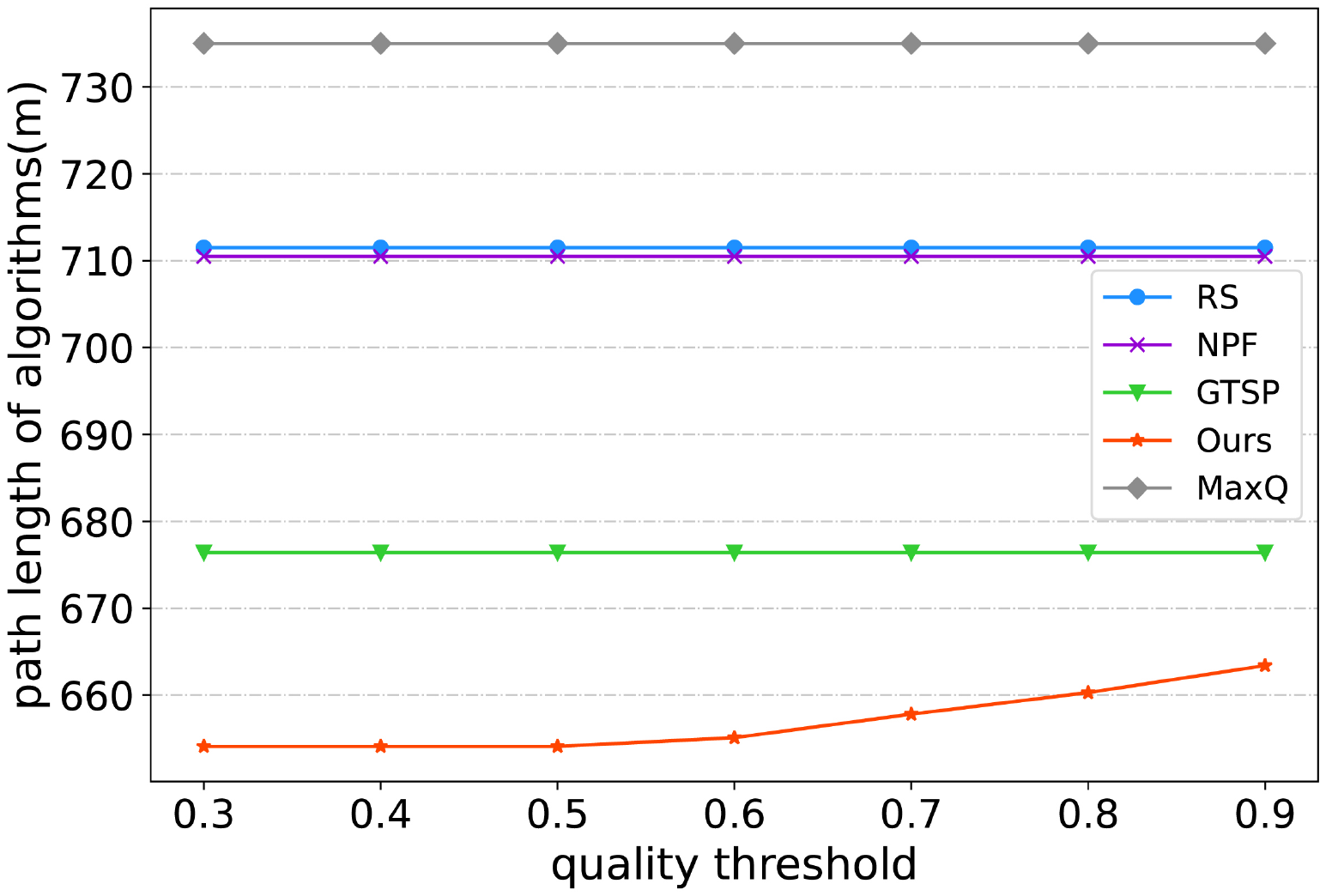}
        \captionsetup{justification=centering}
        \caption{object number = 15}
        \label{fig:pathlength-15}
    \end{subfigure}
    \begin{subfigure}{0.32\textwidth}
        \centering
        \includegraphics[width=\linewidth]{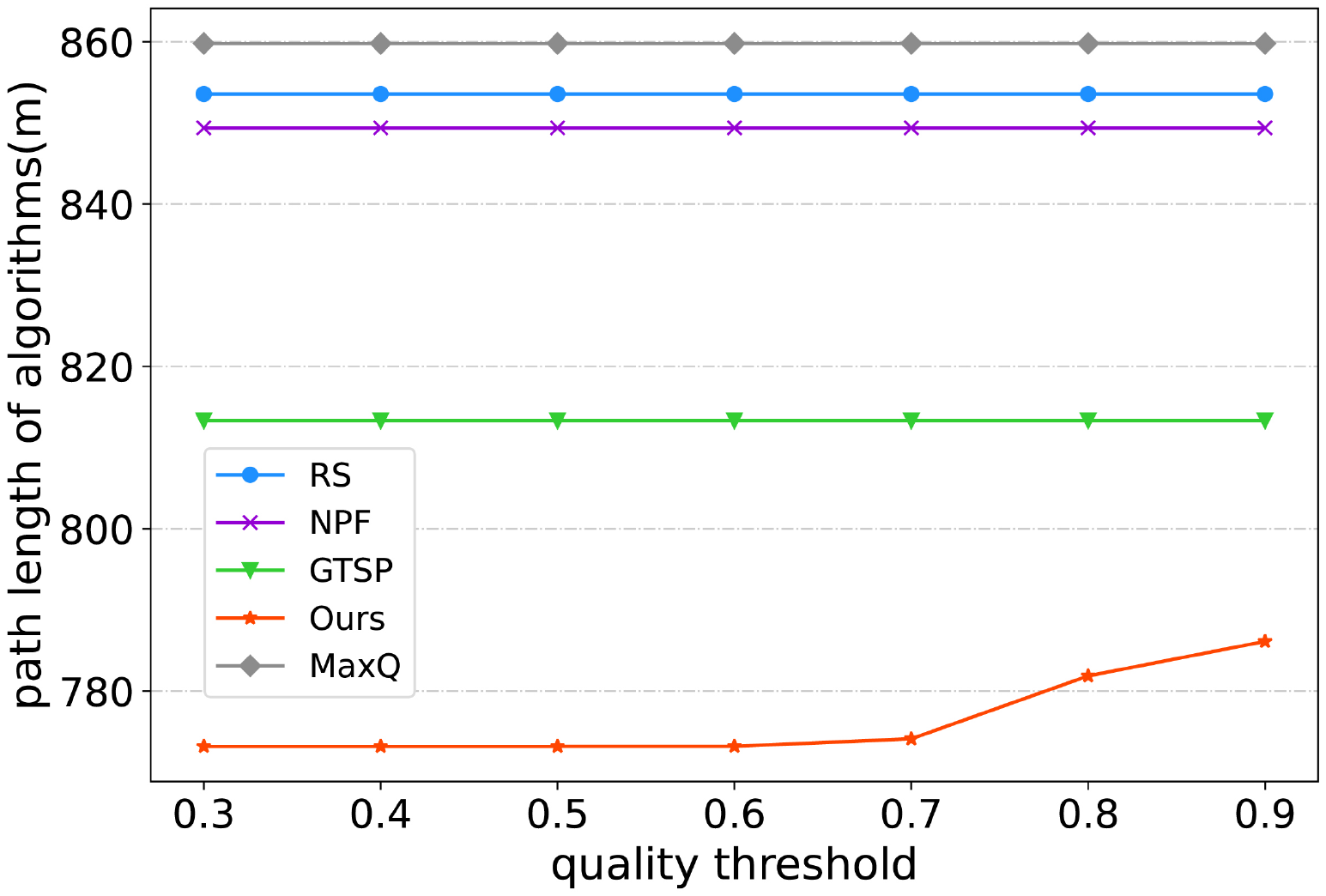}
        \captionsetup{justification=centering}
        \caption{object number = 25}
        \label{fig:pathlength-25}
    \end{subfigure}
\end{minipage}
\captionsetup{justification=centering}
\caption{Path length of algorithms under different observation quality constraints}
\label{fig:path_length_of_algorithms}
\end{figure*}

\subsubsection{Simulation results}

The results of the first simulation experiments are given first. In Figure~\ref{fig:simu_path_length}, we compare the path lengths of our algorithms with that of the baselines. The cases from RS, NPF, GTSP are all divided into two parts, based on the metrics that whether a path has already achieved the quality requirements. Figure~\ref{fig:reduced_path_length} presents the percentage of path lengths that our algorithm can reduce compared to the baseline algorithm.

\begin{figure*}[ht]
\centering
\begin{minipage}[b]{1\linewidth}
    \begin{subfigure}{0.32\textwidth}
        \centering
        \includegraphics[width=\linewidth]{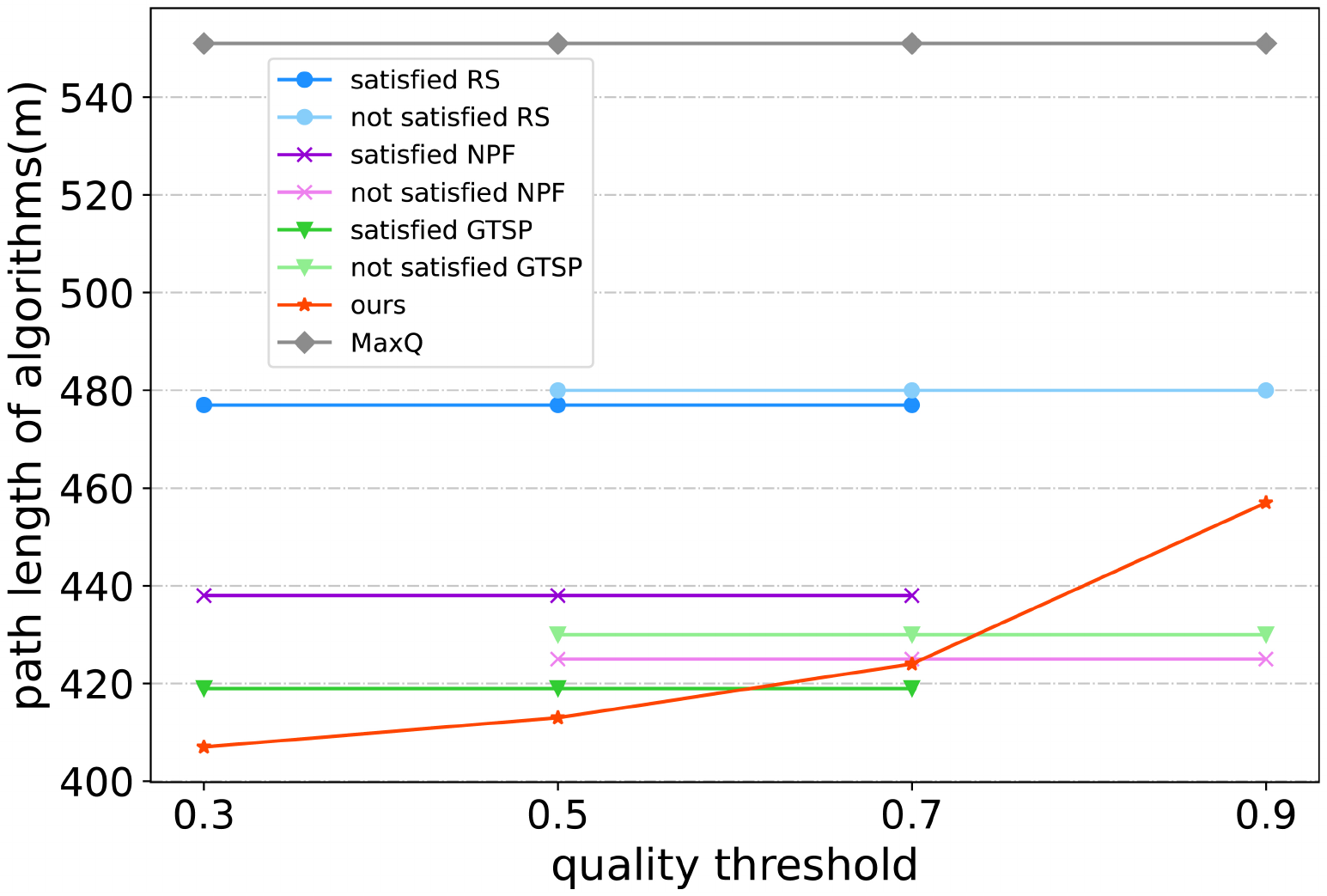}
        \captionsetup{justification=centering}
        \caption{path length, $n$ = 5}
        \label{fig:simu_pathlength5}
    \end{subfigure}
    \begin{subfigure}{0.32\textwidth}
        \centering
        \includegraphics[width=\linewidth]{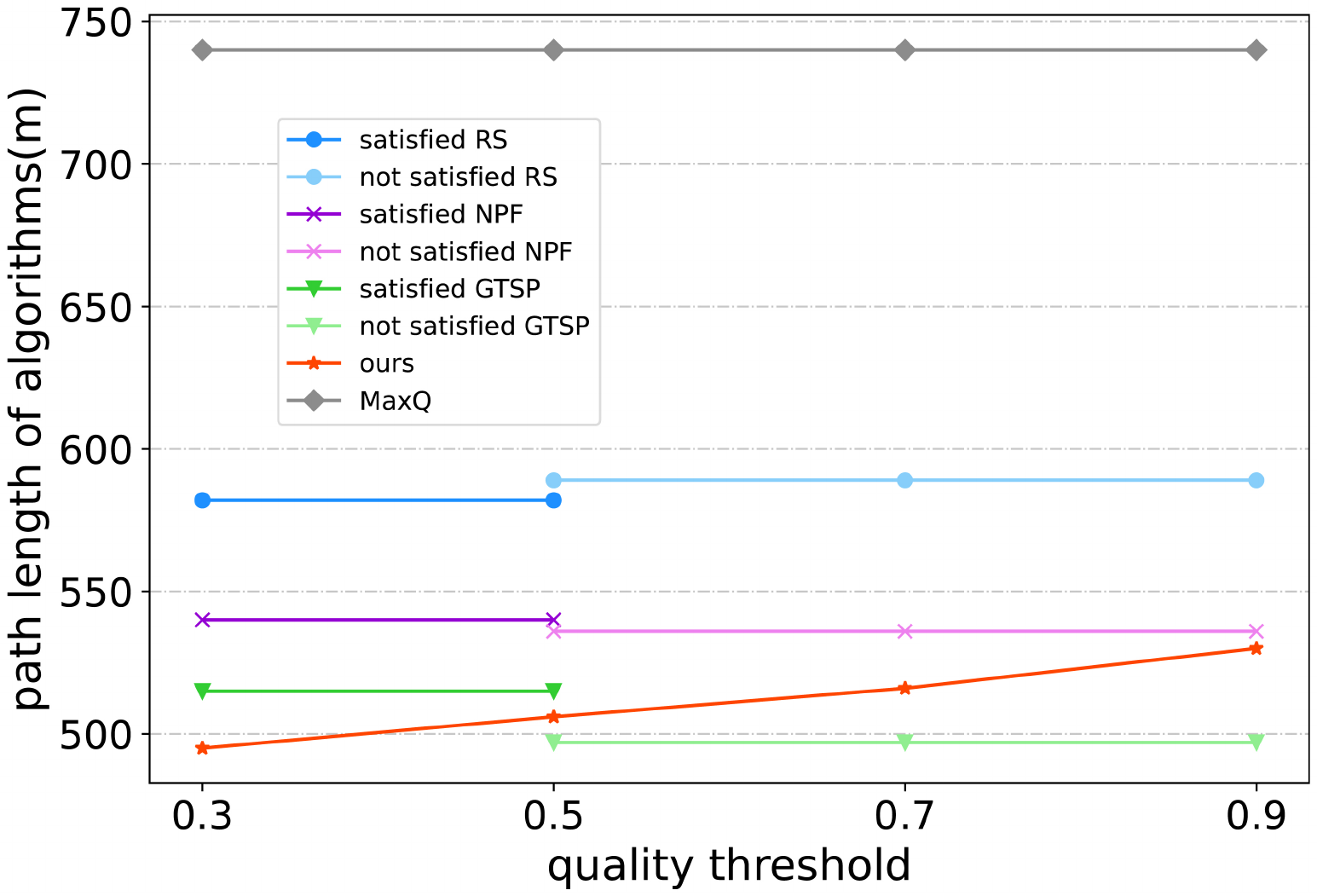}
        \captionsetup{justification=centering}
        \caption{path length, $n$ = 10}
        \label{fig:simu_pathlength10}
    \end{subfigure}
    \begin{subfigure}{0.32\textwidth}
        \centering
        \includegraphics[width=\linewidth]{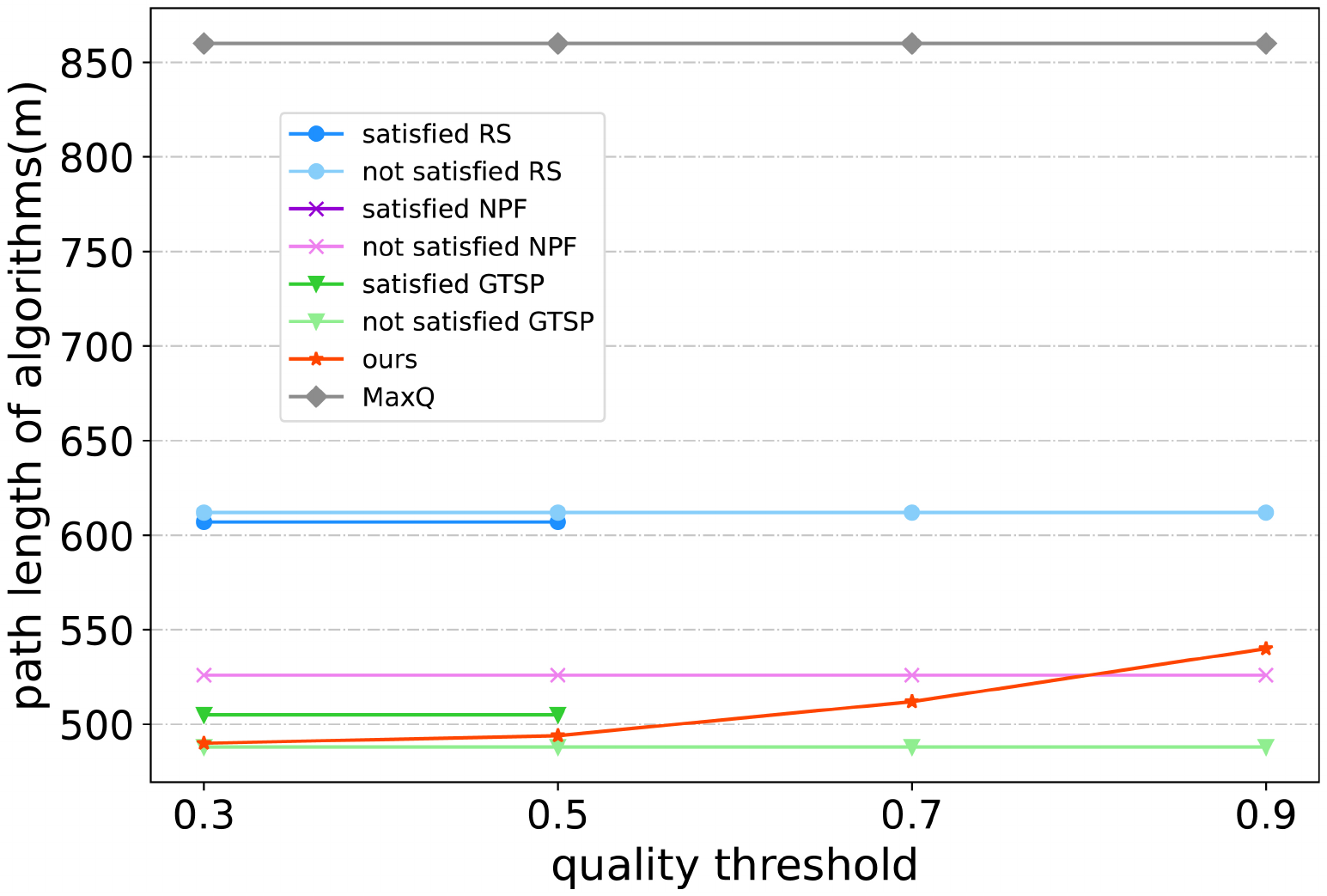}
        \captionsetup{justification=centering}
        \caption{path length, $n$ = 20}
        \label{fig:simu_pathlength20}
    \end{subfigure}
\end{minipage}
\captionsetup{justification=centering}
\caption{Path length of algorithms in the cylinder situation under different observation quality constraints}
\label{fig:simu_path_length}
\end{figure*}

\begin{figure*}[ht]
\centering
\begin{minipage}[b]{1\linewidth}
    \begin{subfigure}{0.32\textwidth}
        \centering
        \includegraphics[width=\linewidth]{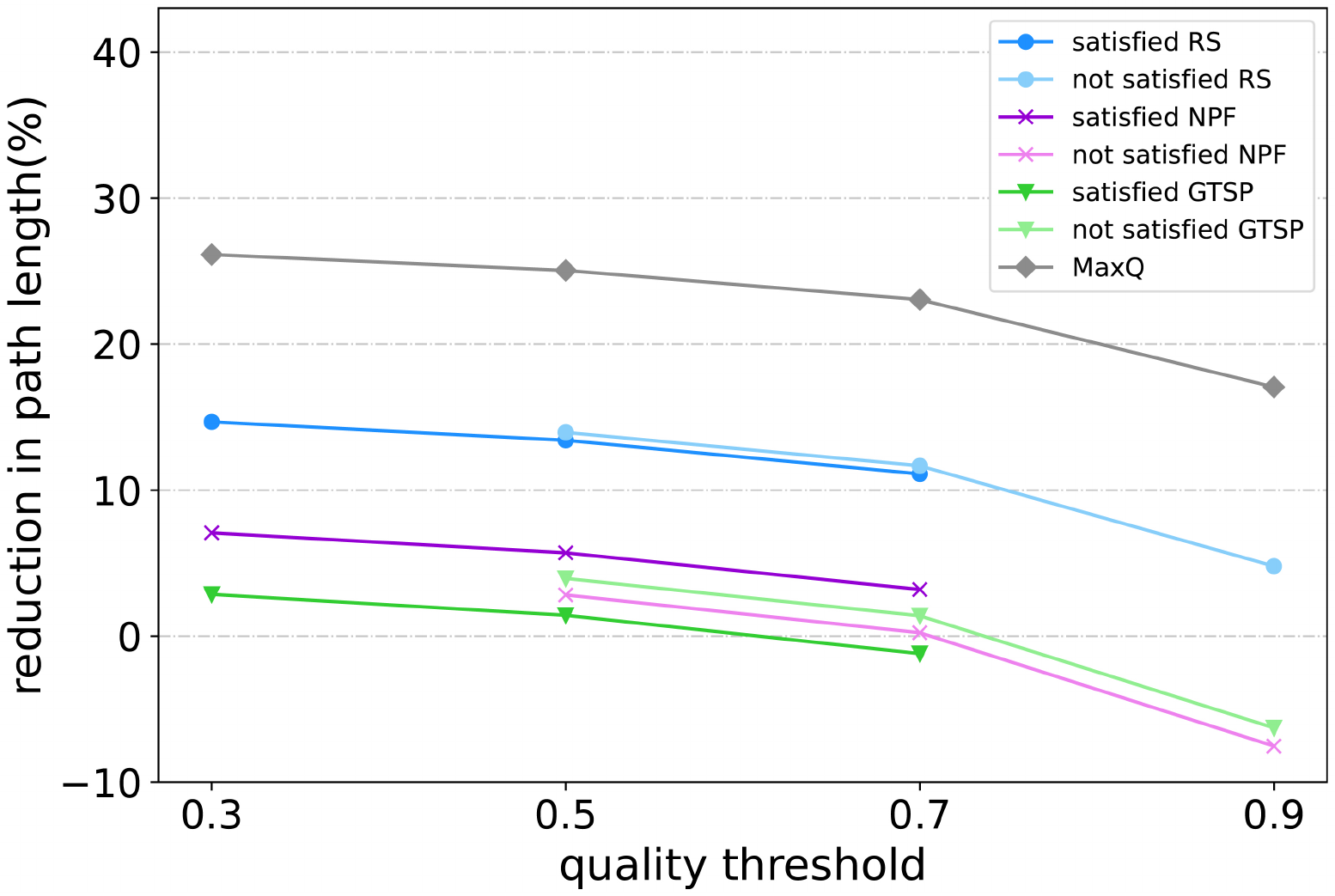}
        \captionsetup{justification=centering}
        \caption{path length, $n$ = 5}
        \label{fig:simu_reduced5}
    \end{subfigure}
    \begin{subfigure}{0.32\textwidth}
        \centering
        \includegraphics[width=\linewidth]{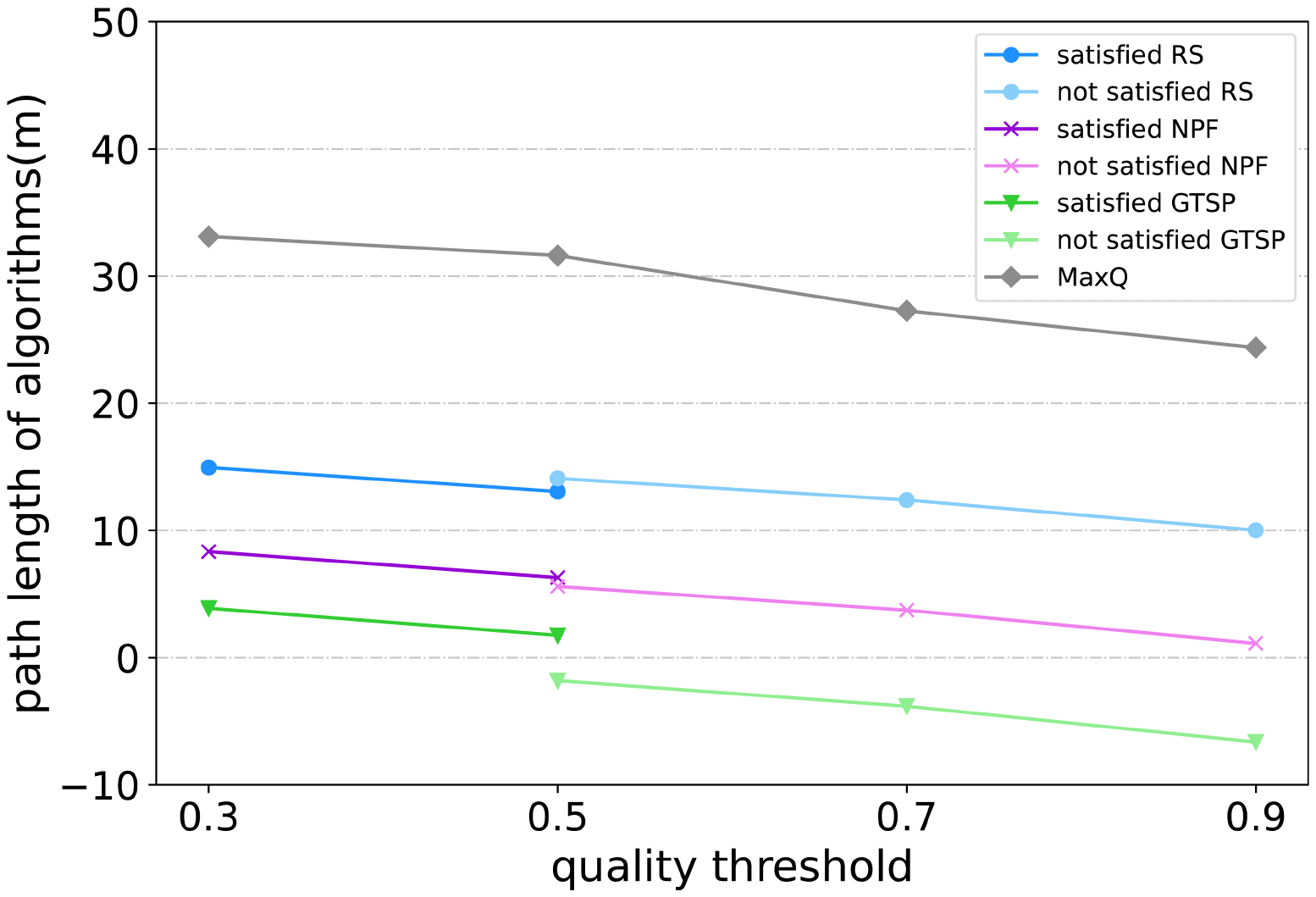}
        \captionsetup{justification=centering}
        \caption{path length, $n$ = 10}
        \label{fig:simu_reduced10}
    \end{subfigure}
    \begin{subfigure}{0.32\textwidth}
        \centering
        \includegraphics[width=\linewidth]{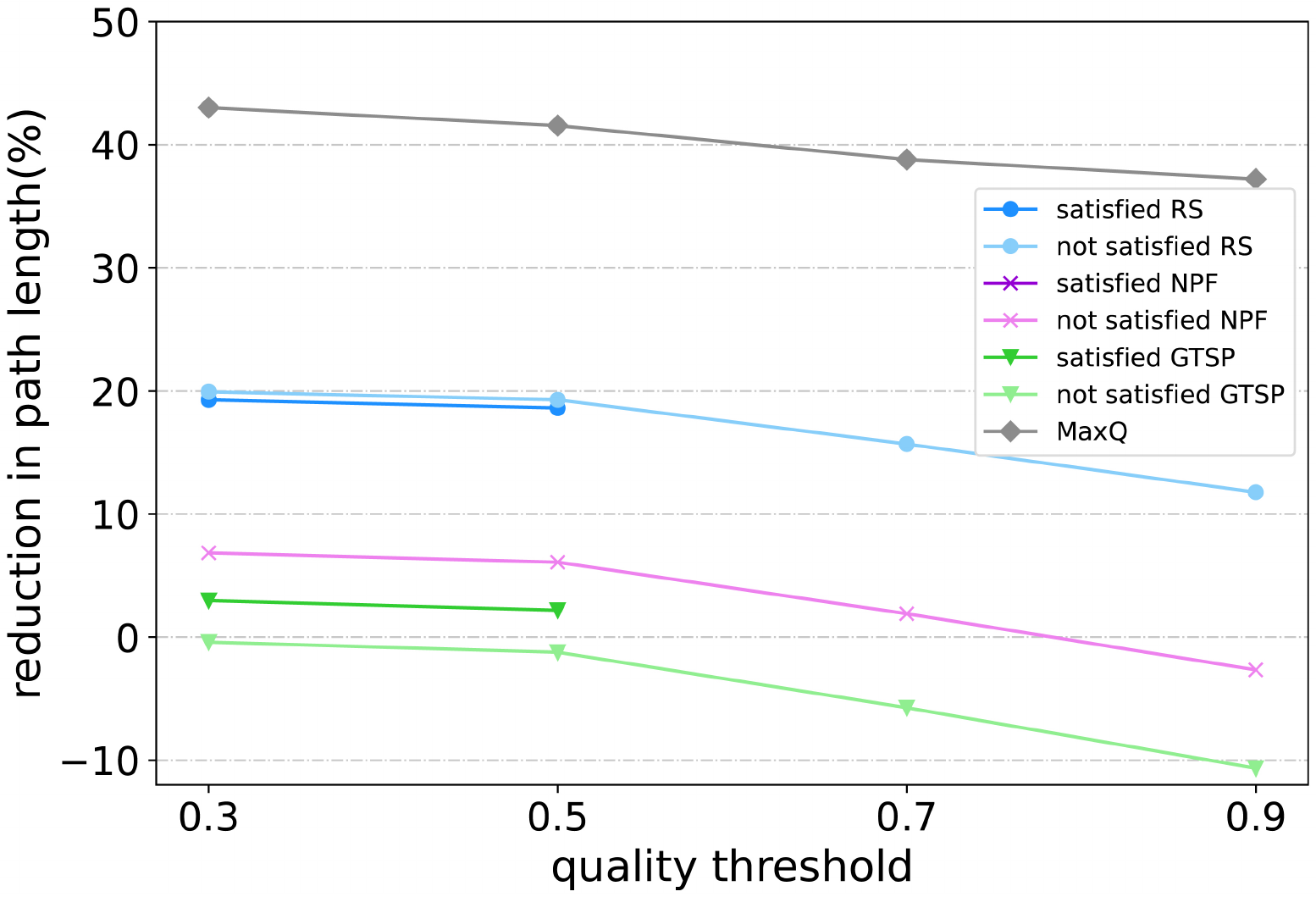}
        \captionsetup{justification=centering}
        \caption{path length, $n$ = 20}
        \label{fig:simu_reduced20}
    \end{subfigure}
\end{minipage}
\captionsetup{justification=centering}
\caption{Reduction percentage of baseline path length in the cylinder situation}
\label{fig:reduced_path_length}
\end{figure*}

It is apparent that the MaxQ path remains the longest and is significantly longer than other paths, with differences ranging from 100 m to 300 m. Since the RS, NPF, and GTSP algorithms do not consider the observation quality, they are represented as horizontal lines in Figure~\ref{fig:simu_path_length}. As $q^*$ increases, the path lengths of our algorithm gradually increase. When $n = 5$, the path length increases from 410 m to 459 m, When $n = 10$, it increases from approximately 490 m to 530 m. Correspondingly, the percentage of reduction in path length also gradually decreases. For baseline paths that already meet the observation quality requirements, our algorithm has a more pronounced effect in reducing the length of RS paths, about 15\%, compared to reductions of approximately 6\% for NPF paths and 3\% for GTSP paths. In cases where the baseline path does not attain the observation quality, our algorithm still functions to a lesser extent, regardless of the negative values for larger $q^*$.
This indicate that our algorithm compensates or the deficiencies in observation quality of the baseline paths at a modest cost of less than 10\%. Compared to the MaxQ path, when $n = 5$, the reduction percentage further increases to about 25\%, and for larger $n$, it ranges from 30\% to 40\%.

The average word recognition accuracy is shown in Figure~\ref{fig:simu_aver_recog_accur} for each algorithm at different quality thresholds. The straight and horizontal lines at the top of each subplot represent the accuracy of the maximum quality paths (MaxQ). This value, $\eta$, does not reach 1 because there probably exists oscillation or other noise factors that may exert influence on photos taken and OCR processes. However, when comparing the average recognition accuracy of each of algorithms with the maximum one, the algorithms' accuracy is extremely close to the product of maximum accuracy and the corresponding quality threshold. As an example, when $q^*$ is 0.9, the algorithms' outcomes are around 0.8 which is approximately $\eta*q^*$.

\begin{figure*}[ht]
\centering
\begin{minipage}[b]{1\linewidth}
    \begin{subfigure}{0.32\textwidth}
        \centering
        \includegraphics[width=\linewidth]{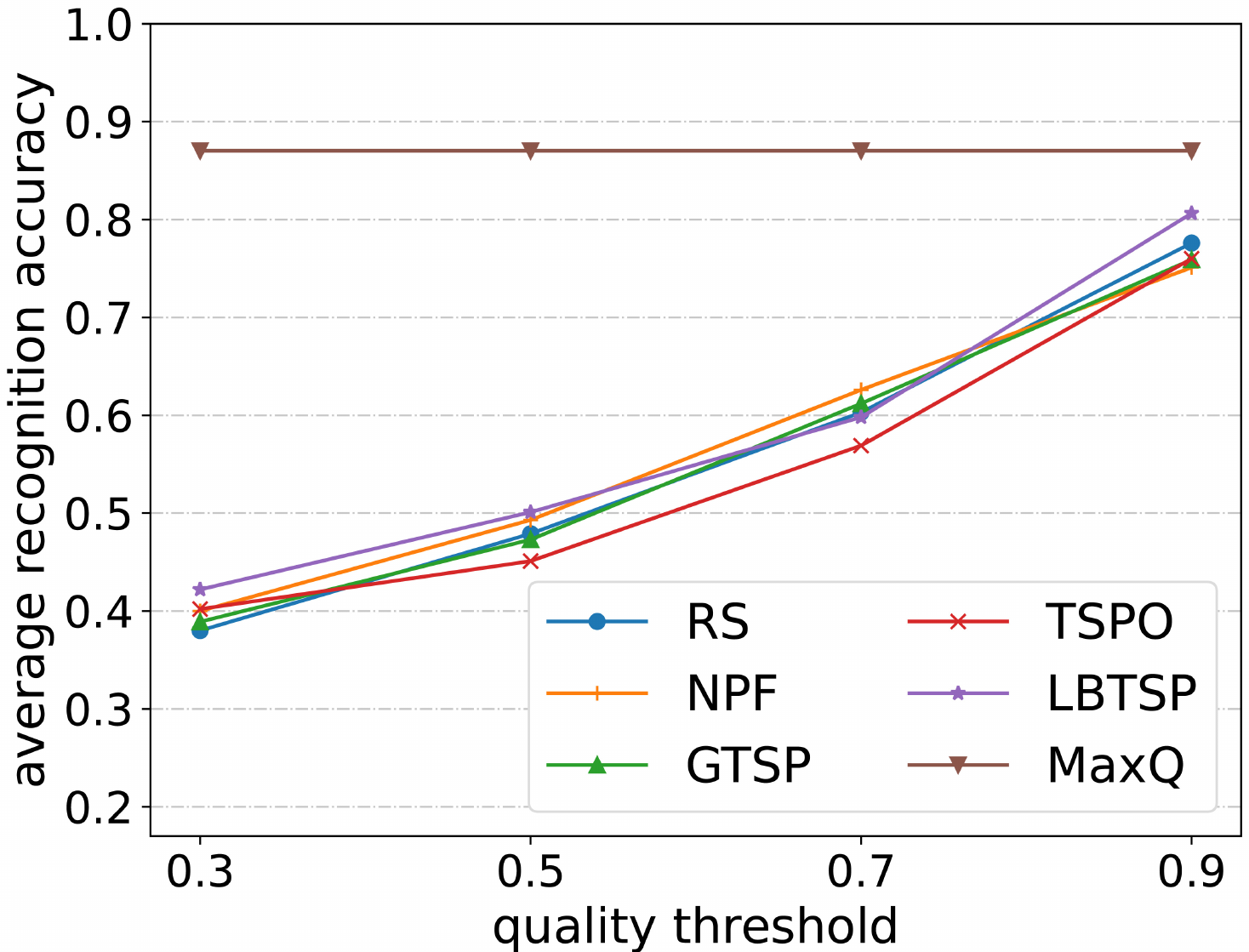}
        \captionsetup{justification=centering}
        \caption{object number = 5}
        \label{fig:aver_5}
    \end{subfigure}
    \begin{subfigure}{0.32\textwidth}
        \centering
        \includegraphics[width=\linewidth]{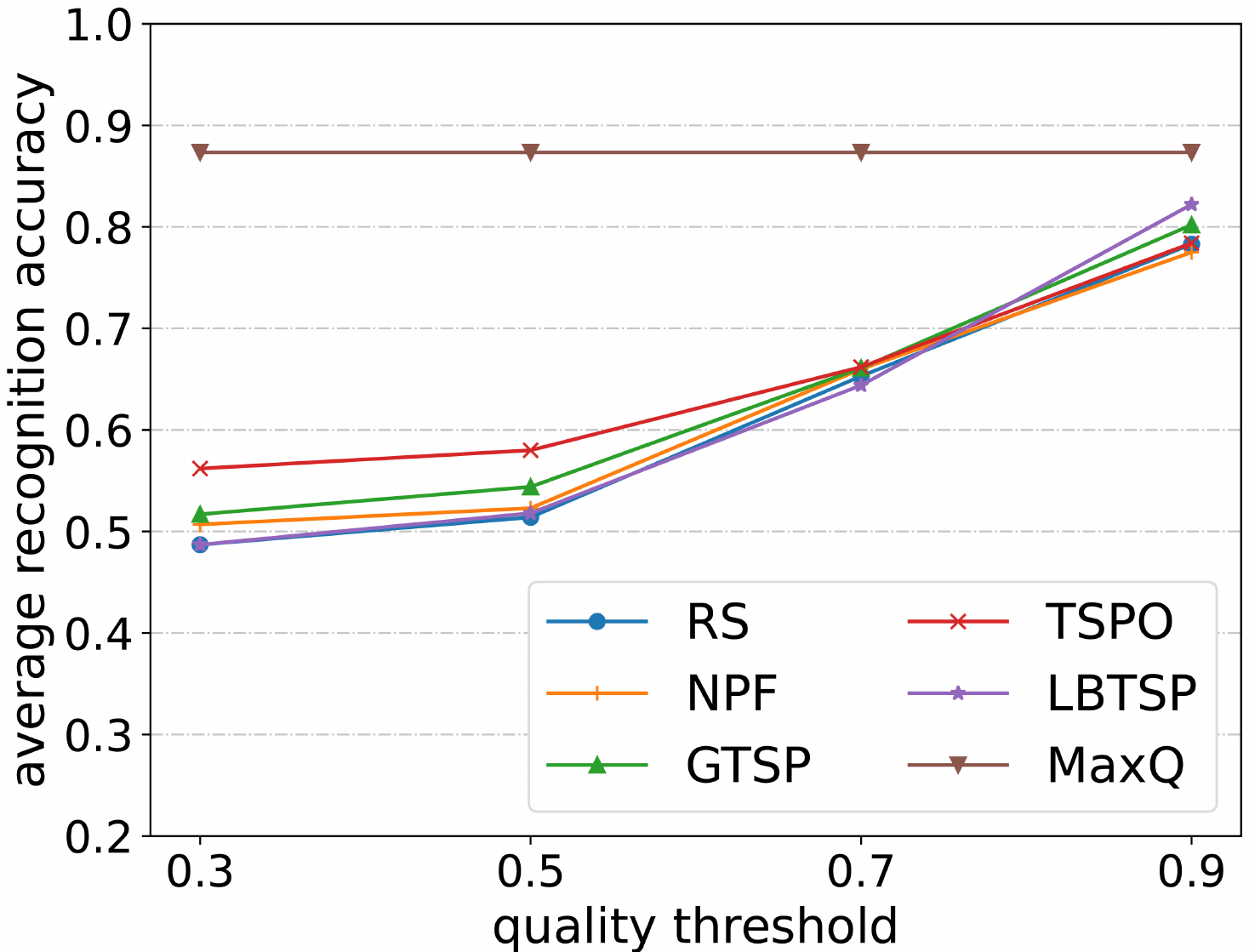}
        \captionsetup{justification=centering}
        \caption{object number = 10}
        \label{fig:aver_10}
    \end{subfigure}
    \begin{subfigure}{0.32\textwidth}
        \centering
        \includegraphics[width=\linewidth]{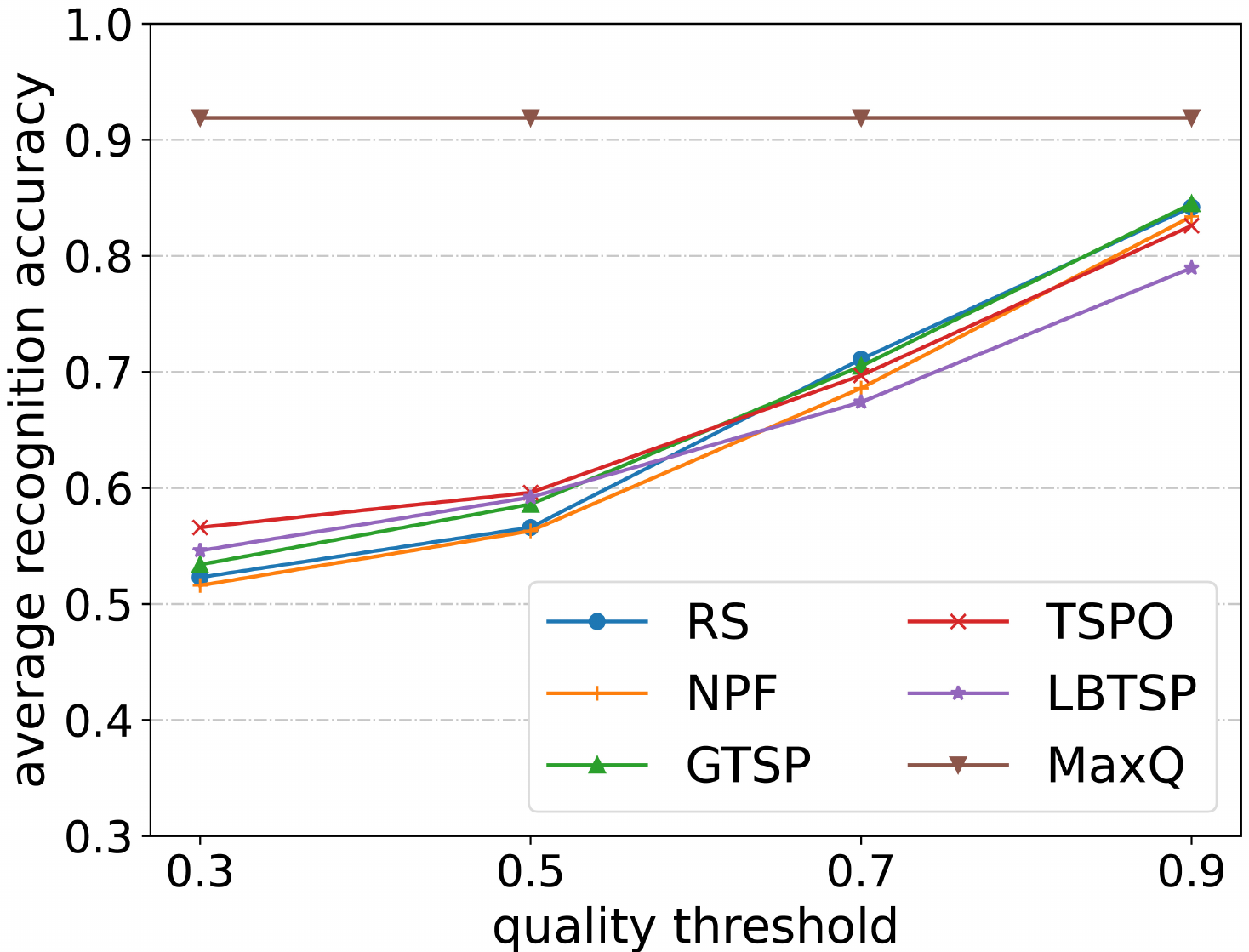}
        \captionsetup{justification=centering}
        \caption{object number = 20}
        \label{fig:aver_20}
    \end{subfigure}
\end{minipage}
\captionsetup{justification=centering}
\caption{Average recognition accuracy at different quality thresholds and object numbers}
\label{fig:simu_aver_recog_accur}
\end{figure*}

The simulation results in the virtual city environment are shown in Figure~\ref{fig:real-city-results}. The trends in path length and reduction ratios exhibit clear similarities with those discussed above. The reduction on the initial path length for RS is the highest, approximately 20\%. In comparison, the reduction in NPF paths are around 7\%, and about 5\% on GTSP paths. Even with high-quality requirements of $q^*=0.9$, our algorithm increases the path length by only about 5\% to meet the observation quality constraint (still achieving 10\% on RS paths). Conversely, the reduction percentage on MaxQ path lengths dramatically increases to around 40\%. Furthermore, Figure~\ref{fig:accu_real_city} shows that the average recognition accuracy is slightly higher than the results in the cylinder environment. The accuracy is approximately 65\% when $q^* = 0.3$, and gradually rises to about 85\% when $q^* = 0.9$. The reason is that objects in the 3D city environments are predominantly planar, thus the impact of deviating angles on perception quality is diminished, making it generally easier to achieve higher recognition accuracy.

\begin{figure*}[ht]
\centering
\begin{minipage}[b]{1\linewidth}
	\subfloat[path length of algorithms]{\label{fig:pathlengthcity}
	\includegraphics[width=0.32\textwidth]{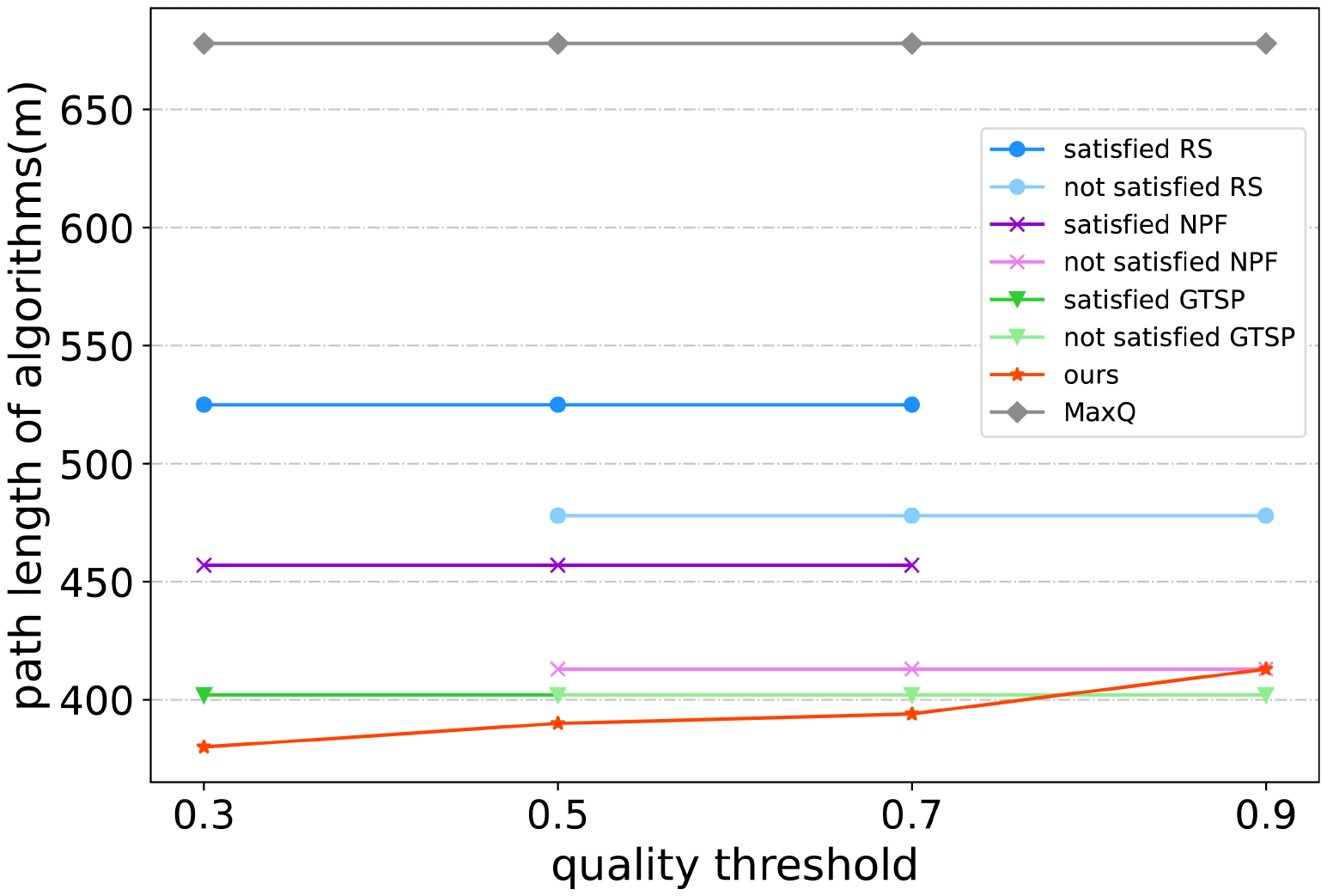}}
	\subfloat[reduction percentage of path length]{\label{fig:reducedcity}
	\includegraphics[width=0.32\textwidth]{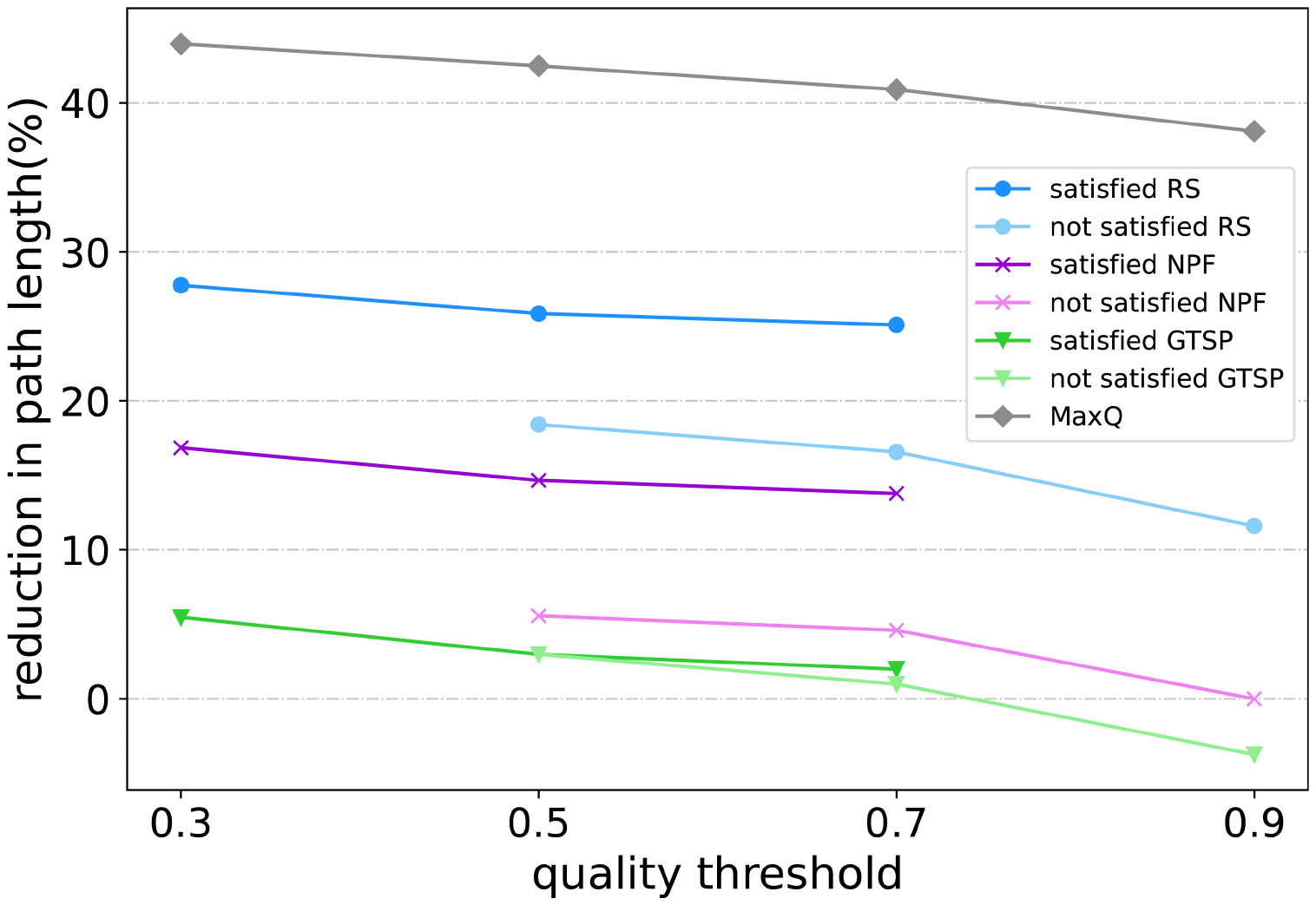}}
        \subfloat[average accuracy]{\label{fig:accu_real_city}
	\includegraphics[width=0.32\textwidth]{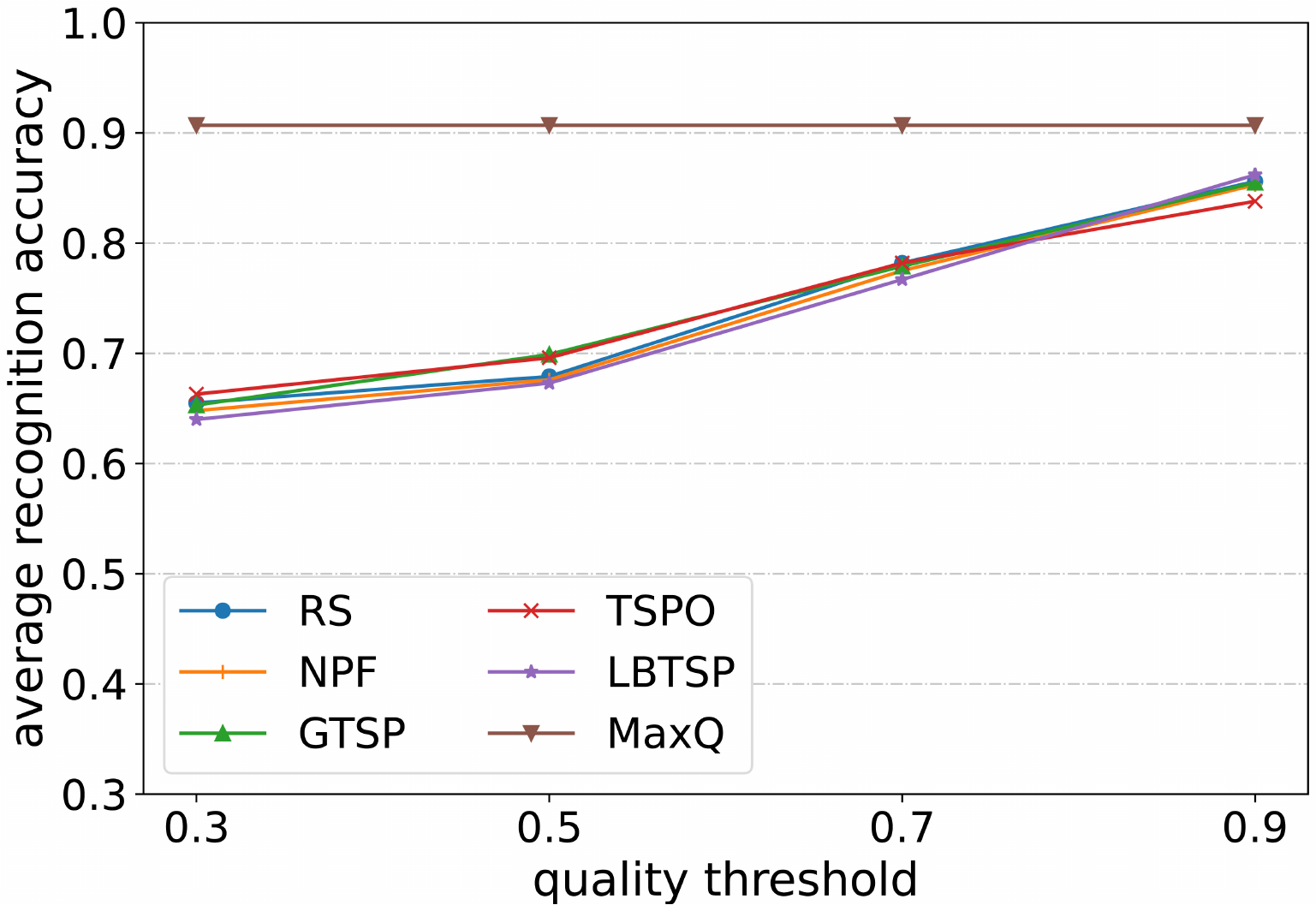}}
\end{minipage}
\caption{Path lengths of algorithms, reduction percentage and recognition accuracy in the 3D city environment}
\label{fig:real-city-results}
\end{figure*}

\subsubsection{Running time} \label{sec:runningtime}
In addition to the approximation ratio and the effectiveness of the dynamic programming algorithm, the average processing time is presented in Table~\ref{table:dp3time}. All of the data units are seconds. Generally, the time cost nearly doubles with the increase of $n$ after 10. When $n$ rises from 5 to 10, it reduces the discretization grid by half, and then the number of observation points of each object quadruples. As the data indicates, there is only minor difference in the running time among different visiting orders, owing to the fact that the running time is only related to the size of the dynamic programming table, which is unique since the problem complexity is the same.

\begin{table}[htb]
\centering
\caption{Running time for dynamic programming algorithm (in seconds), $\epsilon=0.5$}\label{table:dp3time}
\begin{tabular}{ccccccc}
\hline
n & 5  & 10  & 15  & 20 & 25  & 30\\ \hline
time & 14.26 & 211.98 &	723.59	& 1153.01 & 2906.32	& 6874.45\\ \hline
\end{tabular}
\end{table}

\subsection{Discussion}
In general, there is a slightly rising trend in approximation ratio with the increase in the number of objects and observation range. From Theorem~\ref{theorem}, it is clear that larger $\epsilon$ leads to higher approximation ratio, since the error in the grid increases and there are less observation points that can be visited. We suppose that for simple case, the effect of error in fewer visited observation points offset that of Theorem~\ref{theorem}.

Our dynamic programming algorithm can both help the original paths to meet the quality requirement and reduce the lengths of the qualified ones. Provided there are more overlapping observation areas of objects, it can be anticipated a more significant reduction rate on the lengths. In seeking a comparatively good observation quality, there is a worthwhile trade-off on the quality and the path length. The algorithm can seek the path that has $90\%$ observation quality but an evident reduction in the length cost. Based on these results, it can be drawn that GTSP paths yield the relatively best observation order because of the least improvement by the dynamic programming algorithm to achieve a certain quality threshold. Besides, GTSP paths also give the order that contribute to reduce the maximum quality path to the greatest extent, as compared with other heuristic paths. 

Apart from solving the above problem, our algorithm is also endowed with the capability to deal with the situation in which each object has a different weight and various observation quality demands. Suppose the weight of the objects to be ${w_1, w_2, ..., w_n}$ and the quality constraint ${q^*_1, q^*_2, ..., q^*_n}$. The generalization to different weight is achieved by multiplying $w_i$ to the observation quality at all observation points of $o_i$ before determining observing order and running the dynamic programming algorithm. We can eliminate the observation points of $o_i$ that do not yield the required quality $q^*_i$, leaving the rest observation points that all meet the condition. The elimination process can be adapted to different objects under $q^*_i$, hence the second generalization is achieved. Furthermore, in the presence of obstacles, the algorithm still works by adding the edges between the objects and observation points that are not blocked by the obstacles.

\section{Conclusion}
In this work, we focus on planning a path with the lowest length for an UAV that aims to observe a set of objects while ensuring a gross observation quality requirement. Each object is assumed to be a point and faces a certain direction, confining the efficient observation angle and distance. We present a $(1+\epsilon)$-approximation dynamic programming-based algorithm to search for the near optimal path in polynomial time after obtaining an observing order of objects. This is the first work that presents a path planning algorithm that considers the observation quality and has an approximation ratio. Numerical results show that our algorithm achieves near optimal results with the approximation ratio around 1.5 and its effectiveness is also validated in the Airsim simulator.

\bibliographystyle{elsarticle-num} 
\bibliography{biblio}

\begin{thebibliography}{10}
\expandafter\ifx\csname url\endcsname\relax
  \def\url#1{\texttt{#1}}\fi
\expandafter\ifx\csname urlprefix\endcsname\relax\def\urlprefix{URL }\fi
\expandafter\ifx\csname href\endcsname\relax
  \def\href#1#2{#2} \def\path#1{#1}\fi

\bibitem{uavstructureinspection}
W.~Jing, D.~Deng, Y.~Wu, K.~Shimada, Multi-uav coverage path planning for the inspection of large and complex structures, in: 2020 IEEE/RSJ International Conference on Intelligent Robots and Systems (IROS), IEEE, 2020, pp. 1480--1486.

\bibitem{smart-farm}
A.~D. Boursianis, M.~S. Papadopoulou, P.~Diamantoulakis, A.~Liopa-Tsakalidi, P.~Barouchas, G.~Salahas, G.~Karagiannidis, S.~Wan, S.~K. Goudos, Internet of things (iot) and agricultural unmanned aerial vehicles (uavs) in smart farming: A comprehensive review, Internet of Things 18 (2022) 100187.

\bibitem{fire-detect}
C.~Yuan, Z.~Liu, Y.~Zhang, Learning-based smoke detection for unmanned aerial vehicles applied to forest fire surveillance, Journal of Intelligent \& Robotic Systems 93~(1) (2019) 337--349.

\bibitem{uav-cinema}
I.~Mademlis, V.~Mygdalis, N.~Nikolaidis, I.~Pitas, Challenges in autonomous uav cinematography: An overview, in: 2018 IEEE international conference on multimedia and expo (ICME), IEEE, 2018, pp. 1--6.

\bibitem{Lyapunov}
J.~Zhang, J.~Yan, P.~Zhang, X.~Kong, Design and information architectures for an unmanned aerial vehicle cooperative formation tracking controller, IEEE Access 6 (2018) 45821--45833.

\bibitem{math-planning}
J.~De~Waen, H.~T. Dinh, M.~H.~C. Torres, T.~Holvoet, Scalable multirotor uav trajectory planning using mixed integer linear programming, in: 2017 European conference on mobile robots (ECMR), IEEE, 2017, pp. 1--6.

\bibitem{bio-planning}
Z.~Sun, J.~Wu, J.~Yang, Y.~Huang, C.~Li, D.~Li, Path planning for geo-uav bistatic sar using constrained adaptive multiobjective differential evolution, IEEE Transactions on Geoscience and Remote Sensing 54~(11) (2016) 6444--6457.

\bibitem{Kalman-filter1}
M.~Kang, Y.~Liu, Y.~Zhao, A threat modeling method based on kalman filter for uav path planning, in: 2017 29th Chinese Control And Decision Conference (CCDC), IEEE, 2017, pp. 3823--3828.

\bibitem{Kalman-filter2}
Z.~Wu, J.~Li, J.~Zuo, S.~Li, Path planning of uavs based on collision probability and kalman filter, IEEE Access 6 (2018) 34237--34245.

\bibitem{Gaussian-filter}
J.~Yoo, H.~J. Kim, K.~H. Johansson, Path planning for remotely controlled uavs using gaussian process filter, in: 2017 17th international conference on control, automation and systems (ICCAS), IEEE, 2017, pp. 477--482.

\bibitem{k-means-planning}
X.~Yue, W.~Zhang, Uav path planning based on k-means algorithm and simulated annealing algorithm, in: 2018 37th Chinese Control Conference (CCC), IEEE, 2018, pp. 2290--2295.

\bibitem{NN-planning1}
M.~M. Kurdi, A.~K. Dadykin, I.~Elzein, I.~S. Ahmad, Proposed system of artificial neural network for positioning and navigation of uav-ugv, in: 2018 Electric Electronics, Computer Science, Biomedical Engineerings' Meeting (EBBT), IEEE, 2018, pp. 1--6.

\bibitem{NN-planning2}
Y.~Zhang, Y.~Zhang, Z.~Liu, Z.~Yu, Y.~Qu, Line-of-sight path following control on uav with sideslip estimation and compensation, in: 2018 37th Chinese Control Conference (CCC), IEEE, 2018, pp. 4711--4716.

\bibitem{q-learning-planning}
T.~Zhang, X.~Huo, S.~Chen, B.~Yang, G.~Zhang, Hybrid path planning of a quadrotor uav based on q-learning algorithm, in: 2018 37th Chinese control conference (CCC), IEEE, 2018, pp. 5415--5419.

\bibitem{g-learning-planning}
S.~Luan, Y.~Yang, H.~Wang, B.~Zhang, B.~Yu, C.~He, 3d g-learning in uavs, in: 2017 12th IEEE Conference on Industrial Electronics and Applications (ICIEA), IEEE, 2017, pp. 953--957.

\bibitem{DRL}
L.~He, N.~Aouf, B.~Song, Explainable deep reinforcement learning for uav autonomous path planning, Aerospace science and technology 118 (2021) 107052.

\bibitem{DRL2}
R.~Xie, Z.~Meng, L.~Wang, H.~Li, K.~Wang, Z.~Wu, Unmanned aerial vehicle path planning algorithm based on deep reinforcement learning in large-scale and dynamic environments, IEEE Access 9 (2021) 24884--24900.

\bibitem{cellular}
L.~Nam, L.~Huang, X.~J. Li, J.~Xu, An approach for coverage path planning for uavs, in: 2016 IEEE 14th international workshop on advanced motion control (AMC), IEEE, 2016, pp. 411--416.

\bibitem{river-rescue}
P.~Yao, Z.~Xie, P.~Ren, Optimal uav route planning for coverage search of stationary target in river, IEEE Transactions on Control Systems Technology 27~(2) (2017) 822--829.

\bibitem{multi-polygon}
J.~Xie, L.~R.~G. Carrillo, L.~Jin, Path planning for uav to cover multiple separated convex polygonal regions, IEEE Access 8 (2020) 51770--51785.

\bibitem{surveillance}
M.~Petrl{\'\i}k, V.~Von{\'a}sek, M.~Saska, Coverage optimization in the cooperative surveillance task using multiple micro aerial vehicles, in: 2019 IEEE International Conference on Systems, Man and Cybernetics (SMC), IEEE, 2019, pp. 4373--4380.

\bibitem{hetero-cover}
J.~Chen, F.~Ling, Y.~Zhang, T.~You, Y.~Liu, X.~Du, Coverage path planning of heterogeneous unmanned aerial vehicles based on ant colony system, Swarm and Evolutionary Computation 69 (2022) 101005.

\bibitem{seuwsn}
F.~Shan, J.~Huang, R.~Xiong, F.~Dong, J.~Luo, S.~Wang, Energy-efficient general poi-visiting by uav with a practical flight energy model, IEEE Transactions on Mobile Computing (2022).

\bibitem{straightwsn}
J.~Gong, T.-H. Chang, C.~Shen, X.~Chen, Flight time minimization of uav for data collection over wireless sensor networks, IEEE Journal on Selected Areas in Communications 36~(9) (2018) 1942--1954.

\bibitem{bio-inspired}
Q.~Yang, S.-J. Yoo, Optimal uav path planning: Sensing data acquisition over iot sensor networks using multi-objective bio-inspired algorithms, IEEE access 6 (2018) 13671--13684.

\bibitem{robotcover}
V.~An, Z.~Qu, R.~Roberts, A rainbow coverage path planning for a patrolling mobile robot with circular sensing range, IEEE Transactions on Systems, Man, and Cybernetics: Systems 48~(8) (2017) 1238--1254.

\bibitem{wsn}
D.~Alejo, J.~A. Cobano, G.~Heredia, J.~R. Mart{\'\i}nez-de Dios, A.~Ollero, Efficient trajectory planning for wsn data collection with multiple uavs, in: Cooperative Robots and Sensor Networks 2015, Springer, 2015, pp. 53--75.

\bibitem{multi-agent-catastrophe}
D.~Vallejo, J.~J. Castro-Schez, C.~Glez-Morcillo, J.~Albusac, Multi-agent architecture for information retrieval and intelligent monitoring by uavs in known environments affected by catastrophes, Engineering Applications of Artificial Intelligence 87 (2020) 103243.

\bibitem{DFJ-tsp}
G.~Dantzig, R.~Fulkerson, S.~Johnson, Solution of a large-scale traveling-salesman problem, Journal of the operations research society of America 2~(4) (1954) 393--410.

\bibitem{MTZ-tsp}
C.~E. Miller, A.~W. Tucker, R.~A. Zemlin, Integer programming formulation of traveling salesman problems, Journal of the ACM (JACM) 7~(4) (1960) 326--329.

\bibitem{GP-tsp}
L.~Gouveia, J.~M. Pires, The asymmetric travelling salesman problem and a reformulation of the miller--tucker--zemlin constraints, European Journal of Operational Research 112~(1) (1999) 134--146.

\bibitem{1.5TSP}
N.~Christofides, Worst-case analysis of a new heuristic for the travelling salesman problem, Tech. rep., Carnegie-Mellon Univ Pittsburgh Pa Management Sciences Research Group (1976).

\bibitem{nearest-tsp}
M.~Bellmore, G.~L. Nemhauser, The traveling salesman problem: a survey, Operations Research 16~(3) (1968) 538--558.

\bibitem{cheapest-insertion}
T.~Nicholson, A sequential method for discrete optimization problems and its application to the assignment, travelling salesman, and three machine scheduling problems, IMA Journal of Applied Mathematics 3~(4) (1967) 362--375.

\bibitem{genetic-tsp}
J.~H. Holland, Adaptation in natural and artificial systems: an introductory analysis with applications to biology, control, and artificial intelligence, MIT press, 1992.

\bibitem{deploy}
W.~Wang, H.~Dai, C.~Dong, F.~Xiao, J.~Zheng, X.~Cheng, G.~Chen, X.~Fu, Deployment of unmanned aerial vehicles for anisotropic monitoring tasks, IEEE Transactions on Mobile Computing (2020).

\bibitem{approximation-algorithm}
D.~P. Williamson, D.~B. Shmoys, The design of approximation algorithms, Cambridge university press, 2011.

\bibitem{Dijkstra}
E.~W. Dijkstra, et~al., A note on two problems in connexion with graphs, Numerische mathematik 1~(1) (1959) 269--271.

\bibitem{GTSP}
S.~L. Smith, F.~Imeson, Glns: An effective large neighborhood search heuristic for the generalized traveling salesman problem, Computers \& Operations Research 87 (2017) 1--19.

\bibitem{Airsim}
S.~Shah, D.~Dey, C.~Lovett, A.~Kapoor, \href{https://arxiv.org/abs/1705.05065}{Airsim: High-fidelity visual and physical simulation for autonomous vehicles}, in: Field and Service Robotics, 2017.
\newblock \href {http://arxiv.org/abs/arXiv:1705.05065} {\path{arXiv:arXiv:1705.05065}}.
\newline\urlprefix\url{https://arxiv.org/abs/1705.05065}

\bibitem{cnocr}
Breezedeus, \href{https://github.com/breezedeus/CnOCR}{Cnocr: Convolutional recurrent neural networks} (2023).
\newline\urlprefix\url{https://github.com/breezedeus/CnOCR}

\bibitem{unreal-engine}
E.~Games, Unreal engine, \url{https://www.unrealengine.com}, version 4.22.1, released 2019-04-25 (2019).

\bibitem{wtsc}
J.~Li, Y.~Xiong, J.~She, M.~Wu, A path planning method for sweep coverage with multiple uavs, IEEE Internet of Things Journal 7~(9) (2020) 8967--8978.

\end{thebibliography}

\end{document}